\documentclass[journal]{IEEEtran}
\usepackage[ruled,vlined]{algorithm2e}
\SetAlgoSkip{6pt}
\SetAlgoInsideSkip{1pt}
\usepackage{cuted}
\usepackage{mathrsfs}
\usepackage[caption=false]{subfig} 

\usepackage{algcompatible}
\usepackage{algpseudocode}
\algnewcommand\algorithmicinput{\textbf{Input:}}
\algnewcommand\INPUT{\item[\algorithmicinput]}
\algrenewcommand\algorithmicindent{1.2em}
\usepackage{cite}
\usepackage[subtle]{savetrees}
\usepackage{verbatim}
\usepackage{subcaption}
\usepackage{graphicx}
\usepackage{siunitx}
\usepackage{amsmath,amssymb,amsfonts,bm}
\usepackage{epsf}
\usepackage{tikz}
\usepackage{psfrag}
\usepackage{pstool}
\usepackage{epstopdf}
\usepackage{bbm}
\usepackage{booktabs}
\usepackage{stfloats}
\usepackage{url}
\usepackage[dvips]{epsfig}
\usepackage{amsthm}
\usepackage[usenames,dvipsnames]{pstricks}
\usepackage{lettrine}
\usepackage[dvips]{epsfig}
\usepackage{pst-grad}
\usepackage{pst-plot}
\usepackage{epsfig}
\usepackage{enumerate}
\usepackage{amsbsy}
\usepackage{amssymb}
\usepackage{amsthm}
\usepackage{amscd}
\usepackage{float}
\usepackage[caption = false]{subfig}
\usepackage{stackrel}
\usepackage{authblk}
\usepackage{dsfont}
\usepackage[normalem]{ulem}
\usepackage{hyperref}
\usepackage{xcolor}
\usepackage{enumitem}

\makeatletter
\newcommand{\removelatexerror}{\let\@latex@error\@gobble}
\makeatother
\newtheorem{thm}{Theorem}
\newtheorem{lemm}{Lemma}
\newtheorem{cor}{Corollary}
\newtheorem{rema}{Remark}

\usepackage{stackengine}

\definecolor{mintbg}{rgb}{.63,.79,.95}
\begin{document}

\title{{FedCritic-MIMO: Communication-Efficient Serverless Federated Critic Learning for Massive-MIMO Resource Control in Open and Disaggregated 6G RANs}}

\author{Amin Farajzadeh, \textit{Member}, \textit{IEEE}, Melike Erol-Kantarci, \textit{Fellow}, \textit{IEEE}
\thanks{A. Farajzadeh and M. Erol-Kantarci are with the NETCORE Lab, School of Electrical Engineering and Computer Science, University of Ottawa, Ottawa, ON K1N 6N5, Canada (e-mails: \{amin.farajzadeh, melike.erolkantarci\}@uottawa.ca).}}
\pagenumbering{gobble}
\makeatletter
\patchcmd{\@maketitle}
  {\addvspace{0.5\baselineskip}\egroup}
  {\addvspace{-1.8\baselineskip}\egroup}
  {}
  {}
\makeatother

\maketitle

\begin{abstract} 
This paper proposes FedCritic-MIMO, a communication-efficient serverless federated multi-agent reinforcement learning framework for AI-native resource control across independently deployable cell-level controllers in open and disaggregated 6G radio access networks (RANs). We consider a setting in which neighboring controllers do not rely on a common trainer, retain local actors and personalized critic components, and exchange only compatible shared critic parameters. FedCritic-MIMO targets reuse-$1$ multi-cell massive
multiple-input multiple-output (massive-MIMO) orthogonal frequency-division
multiple access (OFDMA) deployments where distributed RAN controllers jointly manage user scheduling, per-stream power allocation, beamforming, interference, and long-term quality-of-service (QoS) under limited inter-controller signaling. Rather than centralizing training or federating actors, each base station (BS) executes its actor locally, while critic knowledge is exchanged through peer-to-peer coordination over an interference-aware graph. To make such collaboration practical, FedCritic-MIMO combines wireless-aware event triggering, adaptive layer-wise top-$k$ sparse critic exchange with error feedback, and balanced interference-aware fusion. We further establish conditional finite-time stationarity and consensus
guarantees for the proposed balanced, compressed peer-to-peer critic
recursion under a fixed-policy, frozen-target critic-regression model. In strongly interference-coupled reuse-$1$ simulations, FedCritic-MIMO achieves the best performance--communication tradeoff among heuristic, independent-learning, centralized-training, and communication-ablation baselines. In particular, it achieves the highest held-out network throughput, improves user-rate distribution and mean signal-to-interference-plus-noise ratio (SINR), increases QoS satisfaction, and attains the lowest interference cost per delivered bit among learning baselines. It also reduces critic-communication overhead by about $76\%$ relative to uncompressed distributed critic exchange. These results demonstrate that serverless exchange of compatible shared critic parameters can coordinate open and disaggregated RAN controllers without centralized trajectory collection, parameter-server aggregation, or actor homogenization.

\end{abstract}

\begin{IEEEkeywords}
AI-native 6G RAN, Open RAN, federated reinforcement learning, multi-agent reinforcement learning, serverless learning, massive MIMO, resource allocation, communication-efficient learning.
\end{IEEEkeywords}

\IEEEpeerreviewmaketitle

\vspace{-2mm}
\section{Introduction}
\label{sec:introduction}

Ultra-dense sixth-generation (6G) radio access networks (RANs), formally associated with
IMT-2030, are expected to support substantially higher capacity, reliability,
connectivity, and intelligence than current cellular systems
~\cite{6G_IMT,6G_intro}. These requirements must be met not only through
new air-interface capabilities, but also through open, disaggregated, and
programmable RAN architectures in which radio-control functions are no
longer confined to a monolithic BS implementation
~\cite{OpenRANArchitecture,OpenRAN-surv}. In such architectures,
neighboring cell-level controllers may be independently deployed, operate
with local observations and implementation-specific model components, and
lack a common centralized trainer. The architectural challenge is therefore
clear: interference-coupled controllers must learn coordinated radio
policies without centralized trajectory collection, unrestricted model
sharing, or periodic aggregation through a parameter server.

This challenge is particularly important for ultra-dense massive
multiple-input multiple-output (massive-MIMO) orthogonal frequency-division
multiple access (OFDMA) systems. Massive MIMO improves spatial
multiplexing, coverage, and interference suppression, while OFDMA provides
flexible time--frequency resource allocation and compatibility with
multiuser scheduling~\cite{OFDMA-intro,6G-MIMO-OFDMA-intro}. This paper
considers a reuse-$1$ multi-cell massive-MIMO OFDMA downlink in which each
BS serves multiple single-antenna user equipments (UEs) over orthogonal
subcarriers and spatially multiplexes several UEs per subcarrier through
space-division multiple access (SDMA).

In this setting, aggressive frequency reuse improves spectral utilization
but creates strong inter-cell interference. At the same time, multiuser
spatial multiplexing introduces intra-cell SDMA interference whenever the
co-scheduled UEs are not perfectly separated by the selected beamformers.
The network performance is therefore jointly determined by user scheduling,
stream activation, per-stream power allocation, beamforming, inter-cell
interference, intra-cell interference, and the evolving quality-of-service (QoS) states of the
UEs. These decisions affect throughput, cell-edge rates, fairness, and
long-term QoS satisfaction, yielding a high-dimensional, mixed
discrete--continuous, and time-varying control problem
~\cite{massive-MIMO-interf-coupl}.

Classical optimization methods provide useful benchmarks, but they commonly
require accurate global channel and interference information, explicit
system models, and repeated solution of mixed-integer nonconvex programs
~\cite{interference_intro}. These requirements are difficult to satisfy in
ultra-dense RANs, where channels, traffic demands, queue states, and
interference patterns vary rapidly. They are also difficult to reconcile
with open and disaggregated control deployments, where local measurements,
experience, actors, and implementation-specific critic components remain
within independently operated cell-level controllers.

Multi-agent reinforcement learning (MARL) is a natural candidate because
each BS can learn from its local channel, queue, and interference
observations~\cite{MARL_intro,MARL_intro-2}. However, reuse-$1$ cells are
not independent learners: the action of one BS changes the interference,
rates, queues, and future scheduling priorities of neighboring cells.
Purely local MARL can therefore learn unstable or poorly coordinated
policies. Centralized reinforcement learning and
centralized-training decentralized-execution (CTDE) can improve
coordination, but typically require joint observations, coordinated
trajectory collection, centralized critics, or common training
infrastructure~\cite{CTDE_intro}. In the considered open and disaggregated
RAN setting, such requirements reintroduce a common trainer across
controllers that are intended to remain independently deployable.

To address this problem, we develop FedCritic-MIMO, a
communication-efficient serverless federated critic framework for
MARL-based joint scheduling, power allocation, and beamforming in
interference-coupled multi-cell massive-MIMO OFDMA networks. Each
cell-level controller executes its actor locally and retains its local
experience, actor parameters, and personalized critic components.
Collaboration is restricted to a predefined compatible shared critic
subnetwork, which is exchanged directly with neighboring controllers over
the interference graph without a central trainer or parameter server.

FedCritic-MIMO makes critic collaboration both learning-aware and
radio-aware. Critic exchange is triggered according to critic innovation,
queue urgency, and interference intensity. Adaptive layer-wise top-$k$
sparsification with error feedback reduces inter-controller model traffic,
while symmetric interference-aware fusion assigns greater weight to
strongly coupled neighbors. The resulting design coordinates distributed
controllers without centralized trajectory aggregation, periodic
full-model synchronization, or homogenization of cell-specific actors and
personalized critic heads.

\vspace{-3mm}
\subsection{Related Work}

\subsubsection{AI-Native Control in Open and Disaggregated RANs}

Open and disaggregated RAN architectures expose programmable control
functions and interfaces for intelligent radio-resource management
~\cite{OpenRANArchitecture}. Recent work has studied learning-enabled
control for O-RAN resource management, including joint scheduling,
O-RU association, and power allocation~\cite{ORAN-PPO}, as well as
federated learning for O-RAN slicing and resource management
~\cite{FDRL-Han}. These works demonstrate the value of
AI-enabled control in programmable RAN architectures. However, they do not
address serverless critic collaboration among independently deployable
cell-level controllers when exchange is restricted to a compatible shared
model component. They also do not jointly consider massive-MIMO
multiuser scheduling, per-stream power allocation, beamforming, dynamic
inter-cell interference, and long-term per-user QoS.

\subsubsection{Learning-Based Wireless Resource Allocation and Beamforming}

Deep reinforcement learning (DRL) and MARL have been widely studied for
dynamic spectrum access, user scheduling, power control, and interference
management~\cite{nasir2019dynamic,naderializadeh2021resource}. These
methods enable wireless controllers to adapt to time-varying channels and
traffic without repeatedly solving nonconvex optimization problems online.
More recent studies have considered higher-dimensional wireless actions,
including joint resource allocation, beamforming, and beam combining
~\cite{dec-MARL-Carleton,Related-noBeam-MARL}, as well as joint
beamforming and subcarrier allocation under queue-aware delay objectives
~\cite{RIS-MARL}. These works are closely related to the physical-layer
control aspect of this paper, but they do not address
communication-efficient serverless critic learning over the physical
interference graph.

\subsubsection{Centralized and Decentralized MARL}

Centralized and CTDE-based MARL methods use centralized critics, joint
observations, value decomposition, or global training signals to improve
coordination among interacting agents~\cite{CTDE-MARL-RA-1,
CTDE-MARL-RA-2}. However, collecting channel, queue, interference,
scheduling, and beamforming information from multiple BSs creates
substantial signaling and scalability requirements
~\cite{Related-CTDE-offline}. More importantly for the considered
architecture, centralized training assumes a common learning function
across cell-level controllers that are otherwise independently deployable.

Fully decentralized MARL removes this dependency by allowing each BS to
learn from local observations~\cite{nasir2019dynamic,
naderializadeh2021resource,dec-MARL}. Decentralized actor--critic methods
can further exchange critic information through graph-based consensus
~\cite{DecMARLTheory,CommEffDecAC}. Nevertheless, unconditional consensus
incurs persistent model traffic, treats neighboring updates without
considering their time-varying radio relevance, and can dilute locally
useful value information under heterogeneous observations, rewards, and
traffic dynamics~\cite{fully-decMARL}. These limitations motivate
selective critic exchange that preserves local actors and personalized
critic components.

\subsubsection{Federated and Communication-Efficient MARL}

Federated learning supports collaborative model training without exchanging
raw local data, while FedAvg and FedProx provide foundational aggregation
mechanisms under statistical and systems heterogeneity
~\cite{FedAvg,FedProx}. Federated MARL has been applied to wireless and
edge control~\cite{ParhizgarOJVT2025}, including channel assignment and
power control~\cite{MFRL-TWC}, decentralized policy collaboration
~\cite{soft-AC}, and communication--computation co-optimization
~\cite{Pareto-AC}. Most existing methods nevertheless rely on a central
parameter server, periodic aggregation, or actor sharing. These mechanisms
can create communication bottlenecks, homogenize policies across
heterogeneous cells, and assume broader model compatibility than is
available in the considered open and disaggregated control setting.

Gossip-based decentralized optimization provides a basis for serverless
model exchange~\cite{CHOCO}, while error feedback compensates for the bias
introduced by sparse compressors such as top-$k$~\cite{EFSGD}. However,
generic gossip and compression mechanisms do not determine when a critic
update is useful for wireless control. In an interference-coupled RAN, the
value of an update depends not only on parameter innovation, but also on
queue urgency and physical interference coupling.

\subsubsection{Positioning of This Work}

General decentralized actor--critic and compressed-consensus methods
provide the algorithmic foundations for peer-to-peer learning. Our earlier
FedCritic framework introduced serverless critic collaboration for
scheduling and power control in multi-cell OFDMA networks
~\cite{FedCritic-gcom}. FedCritic-MIMO addresses the additional
architectural and radio-control problem considered here: collaborative
learning among independently deployable Open RAN cell-level controllers
that lack a common trainer and exchange only a predefined shared critic
subnetwork.

Within this constraint, FedCritic-MIMO jointly addresses multiuser
massive-MIMO scheduling, per-stream power allocation, structured
beamforming, dynamic inter-cell interference, and long-term per-user QoS.
To the best of our knowledge, it is the first framework to combine
communication-efficient serverless critic collaboration, compatibility-
restricted shared-model exchange, interference-aware model fusion,
multiuser massive-MIMO beamforming, and long-term QoS management.

FedCritic-MIMO provides this functionality through: i) local actors and
personalized critic heads, with peer exchange restricted to the shared
critic subnetwork; ii) utility-aware triggering based on critic innovation,
queue urgency, and interference intensity; iii) adaptive layer-wise
top-$k$ exchange with error feedback; and iv) symmetric
interference-aware balanced fusion.

\vspace{-3mm}
\subsection{Contributions}

The main contributions of this paper are summarized as follows:
\begin{itemize}

\item We address the joint improvement of network throughput, long-term QoS
satisfaction, and interference efficiency across independently deployable
cell-level controllers in a reuse-$1$ open and disaggregated
massive-MIMO OFDMA RAN. The scheduling, power-allocation, and beamforming
problem is formulated as an interference-coupled decentralized partially observable Markov
decision process (Dec–POMDP) that captures
multiuser spatial multiplexing, intra-cell and inter-cell interference,
budget-safe power control, structured beamforming, and virtual-queue-based
QoS constraints.

\item We develop a personalized serverless federated critic architecture
without a common trainer or parameter server. Each controller retains its
local actor, experience, and personalized critic head, while only a
predefined compatible shared critic subnetwork is exchanged with
neighboring controllers.

\item We introduce a communication-efficient collaboration mechanism that
combines utility-aware event triggering, adaptive layer-wise top-$k$
sparsification, and error feedback. The triggering utility jointly captures
critic innovation, queue urgency, and interference intensity, while
symmetric interference-aware balanced fusion prioritizes strongly coupled
neighbors and preserves the network-average shared critic update.

\item We establish conditional finite-time stationarity and consensus bounds
for the balanced shared-critic recursion under a fixed policy and
frozen-head, frozen-target critic-regression objective. Under the stated
regularity and tracking assumptions, the analysis yields an
$\mathcal O(T^{-1/2})+\mathcal O(\log T/T)$ randomized-iterate stationarity
rate.

\item We evaluate FedCritic-MIMO against heuristic, independent-learning,
centralized-training, and communication-ablation baselines. FedCritic-MIMO
achieves the strongest held-out throughput, user-rate, signal-to-interference-plus-noise ratio (SINR), QoS, and
interference-efficiency performance among the considered learning methods,
while reducing critic-communication overhead by approximately $76\%$
relative to uncompressed distributed critic exchange.

\end{itemize}

\vspace{-1mm}
\section{System Model}
\label{sec:system_model}

\subsection{Network Setting and Channel Modeling}

We consider the downlink of an open and disaggregated multi-cell massive-MIMO OFDMA RAN with reuse-$1$ frequency allocation. Each cell is managed by an independently deployable programmable RAN controller indexed by the corresponding BS index. The controller performs cell-level radio-resource control using locally available channel measurements, queue states, interference observations, and QoS information. The controllers do not rely on a common trainer or parameter server; local actors, local experience, and personalized critic components remain private, while only compatible shared critic parameters are eligible for peer-to-peer exchange during training. This architectural constraint defines the information and model-sharing structure considered in this work, while the physical-layer signal model follows the standard multi-cell massive-MIMO OFDMA downlink.

The network consists of $N$ BSs sharing a total bandwidth $B$, partitioned
into $K$ orthogonal subcarriers of equal width $\Delta f=B/K$. Each BS
$n\in\mathcal N\triangleq\{1,\ldots,N\}$ is equipped with $L_n$ transmit
antennas and serves a set of associated single-antenna UEs $\mathcal M_n$.
Time is slotted with index $t=0,1,2,\ldots$, and in each slot, the
controller associated with BS $n$ makes online decisions on user scheduling,
linear precoding/beamforming, and transmit-power allocation.

For notational simplicity, ``BS $n$'' refers to BS $n$ together with its
associated local RAN controller unless the distinction is needed explicitly.

The slot duration is normalized to one. If an explicit slot duration $T_s$
is used, the rate expressions below are in bit/s and the corresponding
per-slot service terms used in the virtual queues should be $T_sR_{n,m}(t)$
and $T_sR^{\min}_{n,m}$. Equivalently, one may interpret all rates below as
normalized per-slot service units.

On each subcarrier $k$, BS $n$ may simultaneously multiplex multiple UEs
through linear precoding. Let $x_{n,k,m}(t)\in\{0,1\}$ indicate whether UE
$m\in\mathcal{M}_n$ is scheduled by BS $n$ on subcarrier $k$ in slot $t$.
The intra-cell spatial multiplexing constraint is

\vspace{-1mm}
\begin{equation}\label{eq:sdma_constraint}
\sum_{m\in\mathcal{M}_n} x_{n,k,m}(t) \le S_n, \qquad \forall n,k,
\end{equation}
where $S_n$ denotes the maximum number of simultaneously served streams on each subcarrier at BS $n$, with $1\le S_n \le \min\{L_n, |\mathcal{M}_n|\}$.

For each scheduled UE $m$ on subcarrier $k$, BS $n$ allocates transmit power $p_{n,k,m}(t)\ge 0$ and a unit-norm beamforming vector $\mathbf{v}_{n,k,m}(t)\in\mathbb{C}^{L_n\times 1}$ satisfying
\vspace{-2mm}
\begin{equation}\label{eq:beam_norm}
\|\mathbf{v}_{n,k,m}(t)\|^2 = 1, \qquad \forall n,k,m.
\end{equation}
For unscheduled streams, $p_{n,k,m}(t)$ is forced to zero by the scheduling--power coupling constraint introduced below, and the corresponding beamforming vector is irrelevant to the transmitted signal.
Let $s_{n,k,m}(t)\sim\mathcal{CN}(0,1)$ denote the information symbol intended for UE $m$. The transmitted signal vector from BS $n$ on subcarrier $k$ is
\vspace{-2mm}
\begin{equation}\label{eq:tx_signal_mimo}
\mathbf{u}_{n,k}(t)
=
\sum_{m\in\mathcal{M}_n}
x_{n,k,m}(t)\sqrt{p_{n,k,m}(t)}\,\mathbf{v}_{n,k,m}(t)\,s_{n,k,m}(t).
\end{equation}
Since $\|\mathbf{v}_{n,k,m}(t)\|^2=1$ and $\mathbb{E}[|s_{n,k,m}(t)|^2]=1$, the average transmit power used by BS $n$ on subcarrier $k$ is $\sum_{m\in\mathcal{M}_n} x_{n,k,m}(t)p_{n,k,m}(t)$. The per-BS power budget is therefore
\vspace{-2mm}
\begin{equation}\label{eq:power_budget_mimo}
\sum_{k=1}^{K}\sum_{m\in\mathcal{M}_n} x_{n,k,m}(t)\,p_{n,k,m}(t)\le P_n, \qquad \forall n.
\end{equation}

For any BS $b$ and UE $m\in\mathcal{M}_n$, let $\mathbf{h}_{b,k,m}^{(n)}(t)\in\mathbb{C}^{L_b\times 1}$ denote the downlink channel vector from BS $b$ to UE $m$ associated with BS $n$ on subcarrier $k$ in slot $t$. For the desired link ($b=n$), we use the shorthand $\mathbf{h}_{n,k,m}(t)\triangleq \mathbf{h}_{n,k,m}^{(n)}(t)$.

We model each BS--UE channel vector as
\begin{equation}\label{eq:vector_channel}
\mathbf{h}_{b,k,m}^{(n)}(t)
=
\sqrt{\beta_{b,m}^{(n)}}\,
\big(\bar{\mathbf{R}}_{b,m}^{(n)}\big)^{1/2}
\tilde{\mathbf{h}}_{b,k,m}^{(n)}(t),
\end{equation}
where $\beta_{b,m}^{(n)}>0$ denotes the linear-scale channel-power gain including large-scale fading, $\bar{\mathbf{R}}_{b,m}^{(n)}\succeq \mathbf 0$ is the normalized spatial correlation matrix satisfying $\frac{1}{L_b}\mathrm{tr}(\bar{\mathbf{R}}_{b,m}^{(n)})=1$, and $\tilde{\mathbf{h}}_{b,k,m}^{(n)}(t)\sim\mathcal{CN}(\mathbf 0,\mathbf I_{L_b})$ captures the small-scale fading. With this normalization, $\mathbb E[\|\mathbf h_{b,k,m}^{(n)}(t)\|^2]=L_b\beta_{b,m}^{(n)}$.

The large-scale fading coefficient is modeled in dB as
\vspace{-2mm}
\begin{equation}\label{eq:beta_model}
\beta_{b,m}^{(n)}[\mathrm{dB}]
=
\beta_0
-
10\alpha \log_{10}\!\left(\frac{d_{b,m}^{(n)}}{d_0}\right)
+
F_{b,m}^{(n)},
\end{equation}
where $d_{b,m}^{(n)}$ is the distance between BS $b$ and UE $m$, $\alpha$ is the path-loss exponent, and $F_{b,m}^{(n)}\sim\mathcal N(0,\sigma_{\mathrm{sh}}^2)$ is the shadowing term in dB. The coefficient used in \eqref{eq:vector_channel} is therefore $\beta_{b,m}^{(n)}=10^{\beta_{b,m}^{(n)}[\mathrm{dB}]/10}$.

The temporal evolution of the small-scale fading is modeled by a first-order complex Gauss--Markov process~\cite{Gauss-Markov-Process} as
\vspace{-2mm}
\begin{equation}\label{eq:vector_jakes}
\tilde{\mathbf{h}}_{b,k,m}^{(n)}(t)
=
\rho\,\tilde{\mathbf{h}}_{b,k,m}^{(n)}(t-1)
+
\sqrt{1-\rho^2}\,\mathbf{w}_{b,k,m}^{(n)}(t),
\end{equation}
where $\mathbf{w}_{b,k,m}^{(n)}(t)\sim\mathcal{CN}(\mathbf{0},\mathbf{I}_{L_b})$ is i.i.d. across all indices, and $0\le \rho <1$ is the temporal correlation coefficient. The process is assumed to be initialized with $\tilde{\mathbf h}_{b,k,m}^{(n)}(0)\sim\mathcal{CN}(\mathbf0,\mathbf I_{L_b})$, which preserves the marginal distribution over time. 
The channel is assumed constant within each slot and evolves across slots according to \eqref{eq:vector_jakes}.

The received signal at UE $m\in\mathcal{M}_n$ on subcarrier $k$ in slot $t$ is
\vspace{-2mm}
\begin{align}\label{eq:rx_signal_mimo}
&y_{n,k,m}(t)
=
\underbrace{
\mathbf{h}_{n,k,m}^{H}(t)\,
x_{n,k,m}(t)\sqrt{p_{n,k,m}(t)}\,\mathbf{v}_{n,k,m}(t)\,s_{n,k,m}(t)
}_{\text{desired signal}}
\nonumber\\
&\hspace{-1mm}+
\underbrace{
\sum_{\substack{j\in\mathcal{M}_n\\ j\neq m}}
\mathbf{h}_{n,k,m}^{H}(t)\,
x_{n,k,j}(t)\sqrt{p_{n,k,j}(t)}\,\mathbf{v}_{n,k,j}(t)\,s_{n,k,j}(t)
}_{\text{intra-cell interference}}
\nonumber\\
&\hspace{-1mm} +
{\small \underbrace{
\sum_{b\neq n}\sum_{j\in\mathcal{M}_b}
\big(\mathbf{h}_{b,k,m}^{(n)}(t)\big)^{H}
x_{b,k,j}(t)\sqrt{p_{b,k,j}(t)}\,\mathbf{v}_{b,k,j}(t)\,s_{b,k,j}(t)
}_{\text{inter-cell interference}}
+
z_{n,k,m}(t)},
\end{align}
where $z_{n,k,m}(t)\sim\mathcal{CN}(0,N_0\Delta f)$ is additive noise and $N_0$ denotes the noise PSD under the adopted complex-baseband convention.

The intra-cell and inter-cell interference powers experienced by UE
$m\in\mathcal M_n$ on subcarrier $k$ are respectively defined as
\begin{align}
I^{\mathrm{intra}}_{n,k,m}(t)
&\triangleq
\sum_{\substack{j\in\mathcal{M}_n\\ j\neq m}}
x_{n,k,j}(t)p_{n,k,j}(t)
\big|\mathbf h_{n,k,m}^{H}(t)\mathbf v_{n,k,j}(t)\big|^2,
\label{eq:intra_interference_mimo}\\
I^{\mathrm{inter}}_{n,k,m}(t)
&\triangleq
\sum_{b\in\mathcal N\setminus\{n\}}\sum_{j\in\mathcal M_b}
x_{b,k,j}(t)p_{b,k,j}(t)
\big|
\big(\mathbf h_{b,k,m}^{(n)}(t)\big)^H
\mathbf v_{b,k,j}(t)
\big|^2 .
\label{eq:inter_interference_mimo}
\end{align}
Assuming single-user decoding and treating residual intra-cell and inter-cell
interference as noise, the resulting downlink SINR is
\vspace{-2mm}
\begin{equation}\label{eq:sinr_mimo}
\mathrm{SINR}_{n,k,m}(t)
=
\frac{
x_{n,k,m}(t)p_{n,k,m}(t)
\big|\mathbf h_{n,k,m}^{H}(t)\mathbf v_{n,k,m}(t)\big|^2
}{
I^{\mathrm{intra}}_{n,k,m}(t)
+
I^{\mathrm{inter}}_{n,k,m}(t)
+
N_0\Delta f
}.
\end{equation}
The instantaneous rate of UE $m\in\mathcal{M}_n$ on subcarrier $k$ is
\vspace{-2mm}
\begin{equation}\label{eq:rate_mimo_subcarrier}
R_{n,k,m}^{\mathrm{sc}}(t)=\Delta f\log_2\!\big(1+\mathrm{SINR}_{n,k,m}(t)\big).
\end{equation}
The per-UE, per-subcarrier-per-BS, and per-BS rates are, respectively,
\vspace{-5mm}
\begin{align}
R_{n,m}(t)
&=
\sum_{k=1}^{K} R_{n,k,m}^{\mathrm{sc}}(t), \label{eq:rate_user_mimo}\\
c_{n,k}(t)
&=
\sum_{m\in\mathcal{M}_n} R_{n,k,m}^{\mathrm{sc}}(t), \label{eq:rate_bs_subcarrier_mimo}\\
c_n(t)
&=
\sum_{k=1}^{K} c_{n,k}(t). \label{eq:rate_bs_mimo}
\end{align}
The minimum-rate parameters $R^{\min}_{n,m}$ are assumed to have the same normalized units as $R_{n,m}(t)$ in the virtual-queue update.
\vspace{-2mm}
\subsection{Optimization Problem Formulation}

The long-term QoS requirement for UE $m\in\mathcal M_n$ is
\vspace{-2mm}
\begin{equation}\label{eq:long_term_qos_mimo}
\liminf_{T\rightarrow\infty}
\frac{1}{T}
\sum_{t=0}^{T-1}
\mathbb E\big[R_{n,m}(t)\big]
\ge R^{\min}_{n,m},
\qquad \forall n,\;m\in\mathcal M_n.
\end{equation}
We assume that the vector of minimum-rate requirements is feasible under the available bandwidth, power, scheduling, and beamforming constraints; otherwise, the corresponding virtual queues will diverge and indicate persistent QoS violation.

In each slot $t$, the local cell controllers jointly selects the user scheduling variables $\mathbf{X}(t)=\{x_{n,k,m}(t)\}$, the per-stream powers $\mathbf{P}(t)=\{p_{n,k,m}(t)\}$, and the beamforming vectors $\mathbf{V}(t)=\{\mathbf{v}_{n,k,m}(t)\}$ to maximize the instantaneous network-wide downlink sum-rate while accounting for the long-term minimum-rate requirements. The long-term QoS constraints are handled through virtual queues, which yields the following queue-weighted per-slot surrogate problem:
\vspace{-2mm}
\begin{subequations}\label{prob:sumrate_slot_mimo}
\begin{align}
\hspace{-3mm}\max_{\mathbf{X}(t),\,\mathbf{P}(t),\,\mathbf{V}(t)}
\;
&
\sum_{n=1}^{N}\sum_{k=1}^{K} c_{n,k}(t)
-
\sum_{n=1}^{N}\sum_{m\in\mathcal{M}_n}
Q_{n,m}(t)\Big(R^{\min}_{n,m}-R_{n,m}(t)\Big)
\label{prob:obj_slot_mimo}\\
\text{s.t.}\quad
& \sum_{m\in\mathcal{M}_n} x_{n,k,m}(t)\le S_n,
\qquad  \hphantom{00000000000}\forall n,k, \label{cons:sdma_slot_mimo}\\
& \sum_{k=1}^{K}\sum_{m\in\mathcal{M}_n} x_{n,k,m}(t)\,p_{n,k,m}(t)\le P_n,
\qquad \hphantom{00}\forall n, \label{cons:power_budget_slot_mimo}\\
& 0\le p_{n,k,m}(t)\le P_n\,x_{n,k,m}(t),
\hphantom{00} \forall n,k,\; m\in\mathcal{M}_n, \label{cons:power_coupling_slot_mimo}\\
& \|\mathbf{v}_{n,k,m}(t)\|^2 = 1,
\qquad  \hphantom{0000000}\forall n,k,\; m\in\mathcal{M}_n, \label{cons:beam_norm_slot_mimo}\\
& x_{n,k,m}(t)\in\{0,1\},
\quad  \hphantom{000000000} \forall n,k,\; m\in\mathcal{M}_n. \label{cons:binary_slot_mimo}
\end{align}
\end{subequations}
After normalizing the rate units, \eqref{prob:sumrate_slot_mimo} is the
unit-weight form of the standard drift-plus-penalty surrogate. Equivalently,
one may multiply the sum-rate term by a control parameter $V_{\mathrm{sr}}>0$;
we use $V_{\mathrm{sr}}=1$ throughout. The term
$-\sum_{n,m}Q_{n,m}(t)R^{\min}_{n,m}$ is action-independent within slot $t$, but it is retained to show the deficit-penalty interpretation.
The virtual queues evolve as
\vspace{-2mm}
\begin{equation}\label{eq:vq_update_mimo}
\hspace{-2mm} Q_{n,m}(t{+}1)
=
\Big[\,Q_{n,m}(t)+R^{\min}_{n,m}-R_{n,m}(t)\,\Big]^+,
\ \ \forall n,\; m\in\mathcal{M}_n,
\vspace{-2mm}
\end{equation}
where $[\cdot]^+\triangleq \max\{\cdot,0\}$. Hence, repeatedly solving \eqref{prob:sumrate_slot_mimo} online steers the system toward satisfying long-term average minimum-rate constraints while still prioritizing instantaneous sum-rate. More precisely, stability of the virtual queues implies satisfaction of the corresponding long-term average-rate constraints in \eqref{eq:long_term_qos_mimo}.

Problem \eqref{prob:sumrate_slot_mimo} is a mixed-integer nonconvex program because of the binary scheduling variables, coupled power and beamforming decisions, and reuse-$1$ intra-cell and inter-cell interference. Solving it exactly in every slot would require timely access to network-wide CSI, queue states, QoS states, and interference-coupling information, followed by repeated centralized optimization. This is impractical in large-scale dense deployments and conflicts with the considered open and disaggregated RAN setting, where independently deployable cell-level controllers retain local observations and control components. Purely isolated local control, however, cannot adequately capture the cross-cell effect of each BS's actions on neighboring rates, queues, and interference. These properties motivate the proposed AI-native learning architecture, in which resource-control execution remains local while compatible shared critic parameters are exchanged selectively among interference-neighbor controllers, without centralized trajectory collection or parameter-server aggregation. 

\vspace{-2mm}

\section{Reformulation as an Interference--Coupled Dec--POMDP}
\label{sec:decpomdp}

The online resource-control problem in \eqref{prob:sumrate_slot_mimo} can be reformulated as a cooperative decentralized partially observable Markov decision process (Dec--POMDP), in which each BS acts as an autonomous agent and jointly learns with the other BSs to maximize a common long-term network utility. The resulting Dec--POMDP is interference-coupled, since the rate achieved by each BS depends not only on its own scheduling, power-allocation, and beamforming decisions, but also on the simultaneous decisions made by neighboring BSs over the same subcarriers.
\vspace{-4mm}
\subsection{Dec--POMDP Definition}

We model the system as the tuple
\vspace{-2mm}
\begin{equation}
\mathcal G=\Big\langle
\mathcal N,\mathcal S,\{\mathcal A_n\}_{n\in\mathcal N},
\{\mathcal O_n\}_{n\in\mathcal N},\mathbb P,\mathbb O,r,\mu_0
\Big\rangle,
\label{eq:decpomdp_tuple}
\end{equation}
where $\mu_0$ is the initial-state distribution and the remaining components are
defined below.
\begin{enumerate}[label=(\alph*)]
    \item \emph{Agents}:
The agent set is the set of BSs, $\mathcal{N}=\{1,\dots,N\}$. Each agent $n\in\mathcal{N}$ independently selects its local scheduling, power-allocation, and beamforming decisions at every slot.

\item \emph{Global state}:
The global state at slot $t$ is defined as
\vspace{-2mm}
\begin{equation}
s(t)=\big(\mathbf{H}(t),\mathbf{Q}(t)\big)\in\mathcal{S},
\end{equation}
where $\mathbf{H}(t)
=\{
\mathbf{h}_{b,k,m}^{(n)}(t)
\}$ collects all direct-link and cross-link channel vectors over all BSs, UEs, and subcarriers, and $\mathbf{Q}(t)=\{Q_{n,m}(t)\}$ contains the virtual queues. Under
\eqref{eq:vector_jakes} and \eqref{eq:vq_update_mimo}, $\{s(t)\}$ is Markov.
If lagged measurements are used as policy inputs, they are either included in an
augmented state or handled through the agent's local action--observation history.

\item \emph{Local observations}:
Since BS $n$ does not have access to the full network state, it observes only a local observation $o_n(t)\in\mathcal O_n$. The joint observation is generated according to the observation kernel $\mathbb O(\mathbf o(t)\mid s(t))$, with $\mathbf o(t)=(o_1(t),\ldots,o_N(t))$. A deterministic observation map $o_n(t)=\Omega_n(s(t))$ is a special case.
A natural observation for BS $n$ is
\vspace{-2mm}
\begin{equation}
o_n(t)
=
\Big(
\mathbf{H}^{\mathrm{loc}}_n(t),
\mathbf{Q}_n(t),
\mathbf{I}_n(t)
\Big)\in\mathcal O_n,
\end{equation}
where
$\mathbf{H}^{\mathrm{loc}}_n(t)=
\{
\mathbf{h}_{n,k,m}(t)\}$ contains the direct-link CSI available at BS $n$, 
$\mathbf{Q}_n(t)
=
\{
Q_{n,m}(t)\}$, contains its local virtual queues, and 
$\mathbf{I}_n(t)$ contains only locally available interference-side information before the current action, such as measured interference powers, estimated cross-link CSI from dominant neighboring BSs, or compact neighbor summaries exchanged over the coordination graph. 
Interference created by the simultaneous actions in current slot
$t$ is not assumed known before those actions are selected.

\item \emph{Local action}:
At each slot $t$, BS $n$ chooses an action
\vspace{-2mm}
\begin{equation}
a_n(t)
=
\Big(
\mathbf{X}_n(t),\mathbf{P}_n(t),\mathbf{V}_n(t)
\Big)
\in\mathcal{A}_n,
\end{equation}
where the feasible set is defined by
\eqref{cons:sdma_slot_mimo}--\eqref{cons:binary_slot_mimo} for BS $n$.
Thus, the local action jointly determines: i) which UEs are scheduled on each subcarrier, ii) how much power is allocated to each active stream, and iii) which beamforming/precoding vector is used for each scheduled UE. When a structured rule such as RZF is adopted, $\mathbf{V}_n(t)$ is computed from the scheduled-user CSI and a lower-dimensional learned regularization action. The policy is masked and normalized as specified in Section~\ref{sec:proposed_method} so that sampled actions satisfy the feasibility constraints.

\item \emph{Joint action}:
The joint action of all BSs is
\vspace{-3mm}
\begin{equation}
\mathbf{a}(t)
=
\big(a_1(t),\dots,a_N(t)\big)
\in
\mathcal{A}
\triangleq
\prod_{n=1}^{N}\mathcal{A}_n.
\end{equation}

\item \emph{State transition kernel}:
The transition kernel
$\mathbb{P}\big(s(t+1)\mid s(t),\mathbf{a}(t)\big)$
is induced jointly by the stochastic channel evolution and the deterministic queue update. Specifically, the channel component evolves according to \eqref{eq:vector_jakes}, while the virtual queues evolve as \eqref{eq:vq_update_mimo}. Since the achieved rate $R_{n,m}(t)$ depends on the SINR in \eqref{eq:sinr_mimo}, the queue evolution at BS $n$ is coupled not only to its own action $a_n(t)$, but also to the simultaneous actions of the interfering BSs. 

\item \emph{Rewards}:
The common team reward combines the instantaneous network sum-rate with a
virtual-queue-weighted QoS term that penalizes rate deficits relative to the
minimum-rate requirements and prioritizes users with larger accumulated
deficits. Specifically,

\vspace{-2mm}
\begin{align}
& \hspace*{-2mm} r\big(s(t),\mathbf{a}(t)\big)
=\nonumber \\
& \hspace*{-3mm} \sum_{n=1}^{N}\sum_{k=1}^{K} c_{n,k}(t)
-
\sum_{n=1}^{N}\sum_{m\in\mathcal{M}_n}
Q_{n,m}(t)\Big(R_{n,m}^{\min}-R_{n,m}(t)\Big).
\label{eq:team_reward_mimo}
\end{align}
For decentralized training, BS $n$ uses the corresponding local contribution
to the team reward,
\vspace{-2mm}
\begin{equation}
\hspace*{-3mm} r_n(t)
=
\sum_{k=1}^{K} c_{n,k}(t)
-
\sum_{m\in\mathcal{M}_n}
Q_{n,m}(t)\Big(R_{n,m}^{\min}-R_{n,m}(t)\Big).
\label{eq:local_reward_mimo}
\end{equation}
The local utilities satisfy
$r(s(t),\mathbf a(t))=\sum_{n=1}^{N}r_n(t)$, but each $r_n(t)$ remains
action-coupled because its rates depend on neighboring BS actions through
inter-cell interference and on co-scheduled streams through intra-cell
interference.
\end{enumerate}
\vspace{-5mm}

\subsection{Interference-Coupled Structure}

Let $\mathcal{B}^{\mathrm{int}}_n\subseteq \mathcal{N}\setminus\{n\}$ denote the set of dominant interferers of BS $n$. Then,
the rate of UE $m\in\mathcal{M}_n$ can be approximated as
\vspace{-2mm}
\begin{equation}
R_{n,m}(t)
\approx
\mathcal R_{n,m}^{\mathrm{loc}}
\Big(
s(t),
a_n(t),
\{a_b(t)\}_{b\in\mathcal{B}^{\mathrm{int}}_n}
\Big),
\end{equation}
where $\mathcal R_{n,m}^{\mathrm{loc}}(\cdot)$ denotes the rate mapping
obtained from \eqref{eq:sinr_mimo} after retaining only the dominant
inter-cell interference terms.
This local-interference structure motivates graph-based coordination and neighbor-limited information exchange in the proposed decentralized learning framework.

The Dec--POMDP is therefore interference-coupled in two senses. First, the immediate reward of each BS depends on neighboring actions through the inter-cell interference terms in \eqref{eq:sinr_mimo}. Second, the queue dynamics are coupled across BSs through the achieved rates, since stronger interference from neighboring cells increases the virtual-queue backlog of the affected UEs and thus changes future scheduling priorities.

\vspace{-4mm}
\subsection{Decentralized Control Objective}

Let $\pi_n(a_n\mid o_n)$ denote the stochastic policy of BS $n$, and let the joint policy factorize as
\vspace{-4mm}
\begin{equation}
\boldsymbol{\pi}(\mathbf{a}(t)\mid\mathbf{o}(t))
=
\prod_{n=1}^{N} \pi_n(a_n(t)\mid o_n(t)),
\end{equation}
where $\mathbf{o}(t)=(o_1(t),\dots,o_N(t))$. The decentralized control objective is to learn a set of policies $\{\pi_n\}_{n=1}^{N}$ that maximizes the long-term expected team utility
\vspace{-4mm}

\begin{equation}
\max_{\{\pi_n\}_{n=1}^{N}}
\;
\liminf_{T\rightarrow\infty}
\frac{1}{T}
\mathbb{E}_{\boldsymbol{\pi}}
\left[
\sum_{t=0}^{T-1}
r\big(s(t),\mathbf{a}(t)\big)
\right].
\label{eq:avg_return_mimo}
\end{equation}
Because the virtual queues are part of the system state and the reward in \eqref{eq:team_reward_mimo} is derived from the Lyapunov-drift reformulation, optimizing \eqref{eq:avg_return_mimo} encourages policies that jointly improve network throughput and stabilize the virtual queues, which in turn promotes satisfaction of the long-term minimum-rate requirements.

\vspace{-2mm}

\section{Proposed Communication-Efficient Serverless Federated Critic Learning}
\label{sec:proposed_method}

This section presents the proposed communication-efficient serverless federated critic learning framework for the interference-coupled Dec--POMDP in Section~\ref{sec:decpomdp}. The key idea is to retain the fully decentralized graph-based collaboration structure of FedCritic while replacing periodic full-parameter gossip with a utility-aware and communication-efficient critic-exchange mechanism. In particular, critic communication is driven not only by local critic drift, but also by wireless-control relevance, namely local queue urgency and interference coupling. Each BS maintains its own local actor and critic, performs local actor--critic updates from locally collected trajectories, and exchanges compressed critic-side information only with neighboring BSs over the coordination graph when the local critic update is deemed sufficiently important. This reduces communication overhead while preserving the benefits of serverless federated value learning in the interference-coupled multi-cell massive-MIMO OFDMA setting.

For clarity, we consider synchronous training rounds and reliable peer-to-peer communication over a fixed undirected coordination graph. Critic collaboration is used only during training; after training, each BS executes its local actor using only its local observation. Superscript $t$ denotes a federated training round, whereas $\tau_i$ denotes the $i$th wireless slot in the rollout collected at round $t$.

\vspace{-3mm}
\subsection{Design Rationale}

Compared with the single-antenna OFDMA setting~\cite{FedCritic-gcom}, the massive-MIMO extension substantially enlarges the local state and action spaces due to per-subcarrier multiuser scheduling, per-stream power allocation, and beamforming design. As a result, critic models become larger and more sensitive to local environmental heterogeneity. In such a setting, periodic full-parameter gossip is inefficient, since neighboring BSs may repeatedly exchange highly redundant critic information even when their local critic states have changed only marginally.

More importantly, in interference-coupled wireless control, not all critic updates are equally valuable. A critic change occurring at a BS with low queue pressure and weak interference coupling may have little impact on network performance, whereas even a moderate critic change at a BS operating under high QoS pressure or strong interference coupling may be highly relevant to its neighbors. This motivates a wireless-aware communication policy.

To address these issues, we propose a communication-efficient serverless critic-learning mechanism based on three principles:
\begin{enumerate}
    \item \emph{critic-only federation}: only critic-side information is exchanged across BSs, while actors remain local;
    \item \emph{utility-aware event-triggered communication}: a BS communicates only when its critic update is sufficiently important, as determined jointly by critic innovation, local queue urgency, and interference intensity;
    \item \emph{compressed balanced interference-aware fusion}: when communication is triggered, only a compressed shared-critic increment is exchanged and fused over the interference graph using symmetric interference-relevance weights.
\end{enumerate}

\vspace{-4mm}
\subsection{Local Actor--Critic Parameterization and Objectives}

For BS $n\in\mathcal N$, let $\theta_n^t$ denote the local actor parameters. To allow value-function personalization while preserving parameter-space compatibility, write the critic as
\vspace{-3mm}
\begin{equation}
V_{\vartheta_n^t}(o)
=
h_{\omega_n^t}\!\left(f_{\psi_n^t}(o)\right),
\qquad
\vartheta_n^t=(\psi_n^t,\omega_n^t),
\label{eq:personalized_critic}
\end{equation}
where $\psi_n^t\in\mathbb R^{d_c}$ denotes the shared critic parameters and $\omega_n^t$ denotes a BS-specific value head. All BSs employ compatible shared critic architectures and use the common initialization
\vspace{-2mm}
\begin{equation}
\psi_1^0=\cdots=\psi_N^0=\psi^0.
\label{eq:common_critic_initialization}
\end{equation}
Only $\psi_n^t$ is exchanged; $\theta_n^t$ and $\omega_n^t$ remain local.

BS $n$ uses the local utility in \eqref{eq:local_reward_mimo}, denoted by
$
\bar r_n(\tau)\triangleq r_n(\tau).
$
Its critic approximates the discounted local value
\vspace{-3mm}
\begin{equation}
V_{\vartheta_n}(o_n(\tau))
\approx
\mathbb E_{\boldsymbol\pi}\!\left[
\sum_{\ell=0}^{\infty}
\gamma^\ell\bar r_n(\tau+\ell)
\,\middle|\,
o_n(\tau)
\right],
\label{eq:local_value_function_final}
\end{equation}
where $\gamma\in(0,1)$. Since $o_n(\tau)$ is generally a partial observation, \eqref{eq:local_value_function_final} is a local value approximation rather than the full-state Markov value. The discounted objective is used as a tractable training surrogate for the average-return control objective in \eqref{eq:avg_return_mimo}; no exact equivalence between the two criteria is assumed.

At round $t$, BS $n$ collects the on-policy batch
\vspace{-2mm}
\begin{equation}
\mathcal D_n^t
=
\left\{
\big(
o_n^t(\tau_i),
a_n^t(\tau_i),
\bar r_n^t(\tau_i),
o_n^t(\tau_i+1)
\big)
\right\}_{i=0}^{H_n^t-1}.
\label{eq:local_batch_final}
\vspace{-2mm}
\end{equation}
Using the current critic, define
\vspace{-2mm}
\begin{align}
\delta_{n,i}^t
&=
\bar r_n^t(\tau_i)
+
\gamma V_{\vartheta_n^t}\!\left(o_n^t(\tau_i+1)\right)
-
V_{\vartheta_n^t}\!\left(o_n^t(\tau_i)\right),
\label{eq:td_residual_final}\\
\widehat A_{n,i}^t
&=
\sum_{\ell=0}^{H_n^t-i-1}
(\gamma\lambda)^\ell
\delta_{n,i+\ell}^t,
\qquad
\lambda\in[0,1],
\label{eq:gae_final}\\
\widehat R_{n,i}^t
&=
\widehat A_{n,i}^t
+
V_{\vartheta_n^t}\!\left(o_n^t(\tau_i)\right).
\label{eq:lambda_return_final}
\end{align}
The targets $\{\widehat R_{n,i}^t\}$ are treated as constants during critic optimization. The critic regression loss is
\vspace{-3mm}
\begin{equation}
\mathcal L_n^{\mathrm c}
(\vartheta_n;\mathcal D_n^t)
=
\frac{1}{H_n^t}
\sum_{i=0}^{H_n^t-1}
\left(
V_{\vartheta_n}\!\left(o_n^t(\tau_i)\right)
-
\widehat R_{n,i}^t
\right)^2.
\label{eq:critic_loss_final}
\end{equation}

For the actor, define
\vspace{-4mm}
\begin{equation}
\varrho_{n,i}^t(\theta_n)
=
\frac{
\pi_{\theta_n}\!\left(
a_n^t(\tau_i)\mid o_n^t(\tau_i)
\right)
}{
\pi_{\theta_n^t}\!\left(
a_n^t(\tau_i)\mid o_n^t(\tau_i)
\right)
}.
\label{eq:ppo_ratio_final}
\end{equation}
Then the local PPO loss is calculated as
\vspace{-2mm}
\begin{align}
\mathcal L_n^{\mathrm a}
(\theta_n;\mathcal D_n^t)
={}&
-\frac{1}{H_n^t}
\sum_{i=0}^{H_n^t-1}
\min\!\left\{
\varrho_{n,i}^t(\theta_n)\widehat A_{n,i}^t,
\right.
\nonumber\\[-1mm]
&\left.
\operatorname{clip}\!\left(
\varrho_{n,i}^t(\theta_n),
1-\epsilon_\pi,
1+\epsilon_\pi
\right)
\widehat A_{n,i}^t
\right\}
\nonumber\\
&-
\frac{\beta_H}{H_n^t}
\sum_{i=0}^{H_n^t-1}
\mathcal H\!\left(
\pi_{\theta_n}(\cdot\mid o_n^t(\tau_i))
\right).
\label{eq:actor_loss_final}
\end{align}

After local critic optimization, BS $n$ obtains
$(\widetilde\psi_n^{t+1},\widetilde\omega_n^{t+1})$.
For the convergence analysis, one effective update of the shared critic parameters is written as
\vspace{-2mm}
\begin{equation}
\widetilde\psi_n^{t+1}
=
\psi_n^t-\eta_{\mathrm c}g_n^t,
\qquad
g_n^t
\triangleq
\nabla_{\psi}
\mathcal L_n^{\mathrm c}
(\vartheta_n^t;\mathcal D_n^t),
\label{eq:local_critic_update_final}
\end{equation}
where $\eta_{\mathrm c}>0$. Multiple local critic steps are covered only when their effective direction satisfies the stochastic-gradient conditions stated in Section~\ref{subsec:convergence_analysis}. The actor and the personalized critic head remain local.
\vspace{-3mm}
\subsection{Coordination Graph and Wireless-Relevance Scores}

Let $\mathcal G_{\mathrm c}=(\mathcal N,\mathcal E)$ be the undirected coordination graph and
\vspace{-2mm}
\begin{equation}
\mathcal B_n
=
\left\{
b\in\mathcal N\setminus\{n\}:
(n,b)\in\mathcal E
\right\}
\label{eq:neighbor_set}
\end{equation}
be the neighbor set of BS $n$. Critic collaboration is peer-to-peer; no parameter server is used.

For $b\in\mathcal B_n$, define the interference contributed by BS $b$ to stream $(k,m)$ served by BS $n$ during rollout slot $\tau_i$ as
\vspace{-2mm}
{\small\begin{equation}
\hspace*{-3mm}I_{nb,k,m}^t(\tau_i)
=
\sum_{j\in\mathcal M_b}
x_{b,k,j}^t(\tau_i)
p_{b,k,j}^t(\tau_i)
\left|
\left(
\mathbf h_{b,k,m}^{(n),t}(\tau_i)
\right)^H
\mathbf v_{b,k,j}^t(\tau_i)
\right|^2.\
\label{eq:neighbor_interference_contribution}
\end{equation}}

Let
$\mathcal S_n^t(\tau_i)
=
\left\{
(k,m):
x_{n,k,m}^t(\tau_i)=1
\right\}$.
The normalized queue urgency, total interference intensity, and directional neighbor relevance are
\vspace{-2mm}
{\small \begin{align}
&\bar Q_n^t
=
\frac{1}{H_n^t|\mathcal M_n|}
\sum_{i=0}^{H_n^t-1}
\sum_{m\in\mathcal M_n}
\frac{
Q_{n,m}^t(\tau_i)
}{
Q_{n,m}^t(\tau_i)+q_0
},
\qquad q_0>0,
\label{eq:queue_urgency_score}\\
&\bar I_n^t
=\nonumber\\&
\frac{1}{H_n^t}
\sum_{i=0}^{H_n^t-1}
\frac{1}{
\max\{1,|\mathcal S_n^t(\tau_i)|\}
}
\sum_{(k,m)\in\mathcal S_n^t(\tau_i)}
\frac{
\sum_{b\in\mathcal B_n}
I_{nb,k,m}^t(\tau_i)
}{
\sum_{b\in\mathcal B_n}
I_{nb,k,m}^t(\tau_i)
+
N_0\Delta f
},
\label{eq:local_interference_score}\\
&\bar\kappa_{nb}^t
=\nonumber\\&
\frac{1}{H_n^t}
\sum_{i=0}^{H_n^t-1}
\frac{1}{
\max\{1,|\mathcal S_n^t(\tau_i)|\}
}
\sum_{(k,m)\in\mathcal S_n^t(\tau_i)}
\frac{
I_{nb,k,m}^t(\tau_i)
}{
\sum_{b'\in\mathcal B_n}
I_{nb',k,m}^t(\tau_i)
+
N_0\Delta f
}.
\label{eq:neighbor_interference_relevance}
\end{align}}
An empty scheduled-stream set contributes zero. Hence,
$
0\leq\bar Q_n^t<1
$
and
$
0\leq\bar I_n^t,\bar\kappa_{nb}^t\leq1.
$
\vspace{-3mm}
\subsection{Utility-Aware Event Trigger and Compressed Critic Exchange}

Let $\widehat\psi_n^t$ denote the public reconstruction of BS $n$'s shared critic parameters maintained consistently by BS $n$ and its neighbors. Importantly, $\widehat\psi_n^t$ is a communication-side reference and need not equal the current local critic parameters after neighbor fusion.

The utility-aware trigger score is
\vspace{-2mm}
\begin{equation}
\Gamma_n^t
=
\frac{
\left\|
\widetilde\psi_n^{t+1}
-
\widehat\psi_n^t
\right\|_2
}{
\left\|
\widehat\psi_n^t
\right\|_2
+
\epsilon_{\mathrm{tr}}
}
\left(
1+\alpha_Q\bar Q_n^t+\alpha_I\bar I_n^t
\right),
\label{eq:trigger_score}
\end{equation}
where
$
\epsilon_{\mathrm{tr}}>0
$
and
$
\alpha_Q,\alpha_I\geq0.
$
The communication decision is
\vspace{-2mm}
\begin{equation}
\xi_n^t
=
\mathbf1\left\{
\Gamma_n^t\geq\tau_{\mathrm{th}}^t
\right\},
\label{eq:trigger_rule_final}
\end{equation}
where $\tau_{\mathrm{th}}^t>0$ is a prescribed threshold. A decreasing threshold promotes asymptotically tighter tracking, whereas a positive threshold floor trades residual disagreement for persistent communication savings, as commonly observed in event-triggered learning and federated optimization schemes~\cite{EventFL-TSP}.

Partition the shared critic parameters into $L_c$ layers with dimensions
$
d_1,\ldots,d_{L_c}
$
and
$
\sum_{\ell=1}^{L_c}d_\ell=d_c.
$
At round $t$, layer $\ell$ retains
$
k_{n,\ell}^t
$
coordinates, where
$
1\leq k_{n,\ell}^t\leq d_\ell.
$
Let
$
\mathcal C_{k_{n,\ell}^t}
$
retain the $k_{n,\ell}^t$ largest-magnitude entries of its argument. Define the layer-wise compressor
\vspace{-2mm}
\begin{equation}
\mathcal C_n^t(\mathbf x)
=
\operatorname{col}_{\ell=1}^{L_c}
\left\{
\mathcal C_{k_{n,\ell}^t}(\mathbf x_\ell)
\right\}.
\label{eq:layerwise_compressor}
\end{equation}
It satisfies
\vspace{-3mm}
\begin{equation}
\left\|
\mathcal C_n^t(\mathbf x)-\mathbf x
\right\|_2^2
\leq
(1-\delta_c)\|\mathbf x\|_2^2,
\qquad
\delta_c
\triangleq
\inf_{n,t,\ell}
\frac{k_{n,\ell}^t}{d_\ell}
>0.
\label{eq:compressor_contraction}
\end{equation}

Initialize
$
\widehat\psi_n^0=\widetilde\psi_n^0=\psi^0
$
and
$
\mathbf e_n^0=\mathbf0.
$
Then the variables satisfy
\vspace{-2mm}
\begin{equation}
\mathbf e_n^t
=
\widetilde\psi_n^t-\widehat\psi_n^t.
\label{eq:error_tracking_identity}
\end{equation}
Accordingly, the error-compensated discrepancy and transmitted sparse message are
\vspace{-2mm}
\begin{align}
\mathbf d_n^t
&=
\widetilde\psi_n^{t+1}-\widetilde\psi_n^t+\mathbf e_n^t
=
\widetilde\psi_n^{t+1}
-
\widehat\psi_n^t,
\label{eq:corrected_increment_final}\\
\boldsymbol\Delta_n^t
&=
\xi_n^t
\mathcal C_n^t(\mathbf d_n^t).
\label{eq:compressed_delta_final}
\end{align}
The public reconstruction and residual memory are updated as
\vspace{-2mm}
\begin{align}
\widehat\psi_n^{t+1}
&=
\widehat\psi_n^t
+
\boldsymbol\Delta_n^t,
\label{eq:public_critic_reference_final}\\
\mathbf e_n^{t+1}
&=
\mathbf d_n^t
-
\boldsymbol\Delta_n^t
=
\widetilde\psi_n^{t+1}
-
\widehat\psi_n^{t+1}.
\label{eq:error_feedback_final}
\end{align}
Thus, all unsent coordinates remain in the public-reconstruction error and are reconsidered in later rounds.

\begin{lemm}[One-step tracking-error control]
\label{lem:tracking_error}
For every BS $n$ and round $t$,
\vspace{-4mm}
\begin{equation}
\|\mathbf e_n^{t+1}\|_2^2
\leq
\begin{cases}
\displaystyle
(\tau_{\mathrm{th}}^t)^2
\left(
\|\widehat\psi_n^t\|_2+\epsilon_{\mathrm{tr}}
\right)^2,
& \xi_n^t=0,\\[2mm]
\displaystyle
(1-\delta_c)
\|\mathbf d_n^t\|_2^2,
& \xi_n^t=1.
\end{cases}
\label{eq:tracking_error_bound}
\end{equation}
\end{lemm}
\begin{IEEEproof}
See Appendix~\ref{app:proof_lem1}.
\end{IEEEproof}

\vspace{-2mm}
\subsection{Balanced Interference-Aware Serverless Fusion}

Because $\bar\kappa_{nb}^t$ and $\bar\kappa_{bn}^t$ are generally different, we define the symmetric edge score as
\begin{equation}
s_{nb}^t
=
\epsilon_w
+
\bar\kappa_{nb}^t
+
\bar\kappa_{bn}^t,
\qquad
(n,b)\in\mathcal E,
\label{eq:symmetric_interference_score}
\end{equation}
where $\epsilon_w>0$, and let
$s_n^t
=
\sum_{j\in\mathcal B_n}s_{nj}^t$.
For $b\in\mathcal B_n$, define
\vspace{-2mm}
\begin{equation}
w_{nb}^t
=
(1-\omega_{\mathrm{self}})
\frac{
s_{nb}^t
}{
\max\{s_n^t,s_b^t\}
},
\qquad
0<\omega_{\mathrm{self}}<1,
\label{eq:neighbor_fusion_weights}
\end{equation}
and
\vspace{-2mm}
\begin{equation}
w_{nn}^t
=
1-
\sum_{b\in\mathcal B_n}
w_{nb}^t.
\label{eq:self_fusion_weight}
\end{equation}
All other entries are zero. Since
$
s_{nb}^t=s_{bn}^t,
$
the matrix
$
\mathbf W^t=[w_{nb}^t]
$
is symmetric and row stochastic, hence doubly stochastic. Moreover,
$
w_{nn}^t\geq\omega_{\mathrm{self}}.
$

The shared critic parameters are fused according to
\vspace{-2mm}
\begin{equation}
\psi_n^{t+1}
=
\widetilde\psi_n^{t+1}
+
\sum_{b\in\mathcal B_n}
w_{nb}^t
\left(
\widehat\psi_b^{t+1}
-
\widehat\psi_n^{t+1}
\right),
\label{eq:critic_fusion_final}
\vspace{-2mm}
\end{equation}
while the personalized head remains local as
$\omega_n^{t+1}
=
\widetilde\omega_n^{t+1}$.

\vspace{-3mm}
\subsection{Integrated Critic Recursion}

Define
$\boldsymbol\Psi^t
=
[\psi_1^t,\ldots,\psi_N^t]$,
$\widetilde{\boldsymbol\Psi}^{t+1}
=
[\widetilde\psi_1^{t+1},\ldots,\widetilde\psi_N^{t+1}]$,
$\widehat{\boldsymbol\Psi}^{t+1}
=
[\widehat\psi_1^{t+1},\ldots,\widehat\psi_N^{t+1}]$, and
$\mathbf G^t
=
[g_1^t,\ldots,g_N^t]$. Then~\eqref{eq:critic_fusion_final} is equivalently the following equation:
\vspace{-1mm}
\begin{equation}
\boldsymbol\Psi^{t+1}
=
\widetilde{\boldsymbol\Psi}^{t+1}
+
\widehat{\boldsymbol\Psi}^{t+1}
\left(
(\mathbf W^t)^T-\mathbf I_N
\right).
\label{eq:integrated_critic_rule_final}
\end{equation}
Using
$
\widetilde{\boldsymbol\Psi}^{t+1}
=
\boldsymbol\Psi^t-\eta_{\mathrm c}\mathbf G^t,
$
we obtain
\vspace{-2mm}
\begin{equation}
\boldsymbol\Psi^{t+1}
=
\left(
\boldsymbol\Psi^t-\eta_{\mathrm c}\mathbf G^t
\right)
(\mathbf W^t)^T
+
\mathbf R^t,
\label{eq:perturbed_critic_recursion}
\end{equation}
where
\vspace{-3mm}
\begin{align}
\mathbf R^t
=
\left(
\widetilde{\boldsymbol\Psi}^{t+1}
-
\widehat{\boldsymbol\Psi}^{t+1}
\right)
\left(
\mathbf I_N-(\mathbf W^t)^T
\right).
\label{eq:communication_perturbation}
\end{align}

\begin{lemm}[Average preservation and perturbation control]
\label{lem:average_preservation}
Let
\vspace{-3mm}
\begin{equation}
\varepsilon_t^2
\triangleq
\frac{1}{N}
\mathbb E
\left[
\left\|
\widetilde{\boldsymbol\Psi}^{t+1}
-
\widehat{\boldsymbol\Psi}^{t+1}
\right\|_{\mathrm F}^2
\right]
=
\frac1N\sum_{n=1}^{N}\mathbb E\|\mathbf e_n^{t+1}\|_2^2.
\label{eq:tracking_energy}
\end{equation}
Then
\vspace{-5mm}
\begin{align}
\mathbf R^t\mathbf1
&=
\mathbf0,
\label{eq:average_preserving_error}\\
\bar\psi^{t+1}
&=
\bar\psi^t
-
\eta_{\mathrm c}\bar g^t,
\qquad
\bar g^t
\triangleq
\frac1N\sum_{n=1}^{N}g_n^t,
\label{eq:average_critic_recursion}\\
\frac1N
\mathbb E\|\mathbf R^t\|_{\mathrm F}^2
&\leq
4\varepsilon_t^2.
\label{eq:perturbation_tracking_bound}
\end{align}
\end{lemm}
\begin{IEEEproof}
See Appendix~\ref{app:proof_lem2}.
\end{IEEEproof}
\vspace{-3mm}
\subsection{Conditional Convergence of the Shared Critic Recursion}
\label{subsec:convergence_analysis}

Convergence analysis considers a fixed joint policy $\boldsymbol\pi$ whose induced trajectory process admits a stationary distribution. During the analyzed optimization window, the personalized critic heads and stop-gradient regression targets are held fixed. For BS $n$, define
\vspace{-4mm}
\begin{equation}
F_n(\psi)
=
\mathbb E_{\mathcal D_n\sim d_n^{\boldsymbol\pi}}
\left[
\mathcal L_{n,\mathrm{fr}}^{\mathrm c}
(\psi;\mathcal D_n)
\right],
\quad
F(\psi)
=
\frac1N
\sum_{n=1}^{N}
F_n(\psi),
\label{eq:expected_local_critic_loss}
\end{equation}
where
$
\mathcal L_{n,\mathrm{fr}}^{\mathrm c}
$
denotes the frozen-head, frozen-target loss. Let
\vspace{-3mm}
\begin{equation}
\bar\psi^t
=
\frac1N
\sum_{n=1}^{N}
\psi_n^t,
\qquad
\mathcal E_\psi^t
=
\frac1N
\sum_{n=1}^{N}
\|\psi_n^t-\bar\psi^t\|_2^2.
\label{eq:critic_average_and_consensus}
\end{equation}
\vspace{-1mm}
\noindent\textbf{Assumption 1.}
\begin{enumerate}[label=(A\arabic*)]
\item The matrices $\mathbf W^t$ are symmetric and doubly stochastic, and satisfy the following uniform mixing condition: there exists $\lambda_{\mathrm W}\in[0,1)$ such that
\vspace{-2mm}
\begin{equation}
\left\|
\mathbf W^t
-
\frac1N\mathbf1\mathbf1^T
\right\|_2
\leq
\lambda_{\mathrm W},
\qquad
\forall t.
\label{eq:mixing_assumption}
\vspace{-2mm}
\end{equation}

\item Each $F_n$ is lower bounded and $L$-smooth, i.e.,
\vspace{-2mm}
\begin{equation}
\|\nabla F_n(\psi)-\nabla F_n(\psi')\|_2
\leq
L\|\psi-\psi'\|_2.
\label{eq:smoothness_assumption}
\end{equation}

\item Let $\mathcal F_t$ contain the iterates and all randomness available before the round-$t$ trajectory batches are sampled. The directions satisfy
\vspace{-3mm}
\begin{equation}
\mathbb E[g_n^t\mid\mathcal F_t]
=
\nabla F_n(\psi_n^t).
\label{eq:unbiased_gradient_assumption}
\end{equation}
For finite constants $\sigma_g^2$ and $\sigma_{\mathrm{av}}^2$,
\vspace{-3mm}
\begin{align}
\frac1N
\sum_{n=1}^{N}
\mathbb E\!\left[
\left\|
g_n^t-\nabla F_n(\psi_n^t)
\right\|_2^2
\middle|
\mathcal F_t
\right]
&\leq
\sigma_g^2,
\label{eq:local_noise_variance}\\
\mathbb E\!\left[
\left\|
\frac1N
\sum_{n=1}^{N}
\left(
g_n^t-\nabla F_n(\psi_n^t)
\right)
\right\|_2^2
\middle|
\mathcal F_t
\right]
&\leq
\sigma_{\mathrm{av}}^2.
\label{eq:average_noise_variance}
\end{align}
This formulation permits cross-BS gradient-noise correlation. Under conditional independence,
$
\sigma_{\mathrm{av}}^2\leq\sigma_g^2/N.
$

\item The local-objective heterogeneity is bounded as
\vspace{-2mm}
\begin{equation}
\frac1N
\sum_{n=1}^{N}
\left\|
\nabla F_n(\psi)-\nabla F(\psi)
\right\|_2^2
\leq
\zeta^2,
\qquad
\forall\psi.
\label{eq:heterogeneity_assumption}
\vspace{-2mm}
\end{equation}

\item The shared critic parameters use the common initialization
\eqref{eq:common_critic_initialization}, and
$
0<\eta_{\mathrm c}\leq\eta_{\max},
$
where $\eta_{\max}$ is sufficiently small relative to $L$ and the uniform spectral gap $1-\lambda_{\mathrm W}$.
\end{enumerate}

\begin{thm}
\label{thm1}
Under Assumption~1, there exist constants
$
C_1,\ldots,C_6>0
$
independent of $T$ such that
\vspace{-2mm}
{\small\begin{align}
\frac1T
\sum_{t=0}^{T-1}
\mathbb E
\left[
\|\nabla F(\bar\psi^t)\|_2^2
\right]
\leq{}&
\frac{
C_1\big(F(\bar\psi^0)-F^\star\big)
}{
\eta_{\mathrm c}T
}
+
C_2\eta_{\mathrm c}\sigma_{\mathrm{av}}^2
\nonumber\\
&+
C_3
\frac{
\eta_{\mathrm c}^2
(\sigma_g^2+\zeta^2)
}{
(1-\lambda_{\mathrm W})^2
}
+
\frac{C_4}{(1-\lambda_{\mathrm W})^2}
\frac1T
\sum_{t=0}^{T-1}
\varepsilon_t^2,
\label{eq:critic_stationarity_bound}
\vspace{-2mm}
\end{align}}
where $F^\star$ is a lower bound of $F$. Moreover,
\vspace{-2mm}
\begin{equation}
\frac1T
\sum_{t=0}^{T-1}
\mathbb E[\mathcal E_\psi^t]
\leq
C_5
\frac{
\eta_{\mathrm c}^2
(\sigma_g^2+\zeta^2)
}{
(1-\lambda_{\mathrm W})^2
}
+
\frac{C_6}{(1-\lambda_{\mathrm W})^2}
\frac1T
\sum_{t=0}^{T-1}
\varepsilon_t^2.
\label{eq:critic_consensus_bound}
\vspace{-2mm}
\end{equation}
\end{thm}

\begin{IEEEproof}
See Appendix~\ref{app:proof_thm1}.
\end{IEEEproof}

In this bound, the tracking-error term
$
T^{-1}\sum_{t=0}^{T-1}\varepsilon_t^2
$
plays the role of a communication-induced optimization-error term, consistent with wireless federated-learning analyses in which unreliable or resource-constrained communication appears explicitly in the convergence behavior~\cite{ChenTWC2024RobustFL}.

\begin{cor}
\label{cor:critic_rate}
Let $J_T$ be uniformly distributed over $\{0,\ldots,T-1\}$
and independent of the training randomness. If
\vspace{-2mm}
\begin{equation}
\eta_{\mathrm c}
=
\frac{c_\eta}{\sqrt T}
\vspace{-2mm}
\end{equation}
for sufficiently small $c_\eta>0$ and
\vspace{-2mm}
\begin{equation}
\frac1T
\sum_{t=0}^{T-1}
\varepsilon_t^2
=
\mathcal O\!\left(
\frac{\log T}{T}
\right),
\label{eq:tracking_rate_condition}
\vspace{-2mm}
\end{equation}
then
\vspace{-2mm}
\begin{equation}
\mathbb E
\left[
\|\nabla F(\bar\psi^{J_T})\|_2^2
\right]
=
\mathcal O(T^{-1/2})
+
\mathcal O\!\left(
\frac{\log T}{T}
\right).
\label{eq:critic_convergence_rate}
\end{equation}
For a fixed learning rate and
$
\limsup_{T\to\infty}
T^{-1}\sum_{t=0}^{T-1}\varepsilon_t^2
\leq
\bar\varepsilon^2,
$
the method converges to a stationarity neighborhood whose size is governed by
$
\eta_{\mathrm c}\sigma_{\mathrm{av}}^2
$,
$
\eta_{\mathrm c}^2(\sigma_g^2+\zeta^2)/(1-\lambda_{\mathrm W})^2
$,
and
$
\bar\varepsilon^2/(1-\lambda_{\mathrm W})^2.
$
\end{cor}
\begin{IEEEproof}
See Appendix~\ref{app:proof_cor1}.
\end{IEEEproof}

\begin{rema}
Lemma~\ref{lem:tracking_error} gives the exact one-step effect of event triggering and top-$k$ compression on the public-reference error. Corollary~\ref{cor:critic_rate} deliberately states the additional decay condition on the accumulated tracking error rather than assuming that a particular threshold schedule automatically guarantees it. Establishing \eqref{eq:tracking_rate_condition} requires corresponding control of the shared-critic drift and communication schedule.
\end{rema}

\vspace{-3mm}
\begin{rema}
Theorem~\ref{thm1} concerns first-order stationarity of the fixed-policy, frozen-head, frozen-target aggregate critic objective. It does not establish convergence or global optimality of the evolving PPO actors, the complete actor--critic process, the Dec--POMDP, or the original mixed-integer wireless resource-allocation problem.
\end{rema}
\vspace{-4mm}
\subsection{Communication Overhead}

Let
$
k_n^t
=
\sum_{\ell=1}^{L_c}
k_{n,\ell}^t.
$
If each transmitted value uses $b_v$ bits and a coordinate index uses
$
b_i=\lceil\log_2 d_c\rceil
$
bits, the directed logical critic payload at round $t$ is
\vspace{-4mm}
\begin{equation}
\mathcal B_{\mathrm{comm}}(t)
=
\sum_{n=1}^{N}
|\mathcal B_n|
\xi_n^t
k_n^t
(b_v+b_i)
\quad\text{bits}.
\label{eq:comm_bits_final}
\end{equation}
If one physical broadcast is received by all neighbors, the factor
$
|\mathcal B_n|
$
is omitted for that transmitting BS. Equation~\eqref{eq:comm_bits_final} excludes protocol headers and the scalar relevance summaries required to construct \eqref{eq:symmetric_interference_score}. Because the current balanced weights require both $\bar\kappa_{nb}^t$ and $\bar\kappa_{bn}^t$, these low-rate summaries are exchanged independently of whether the critic trigger fires, or their most recently received values must be used. 
\vspace{-3mm}

\subsection{Training Procedure}

Algorithm~\ref{alg:fedcritic_mappo} summarizes the training procedure. Critic collaboration affects training only; decentralized execution uses the local actor at each BS.

\begin{algorithm}[!t]
\caption{Communication-Efficient FedCritic-MIMO}
\label{alg:fedcritic_mappo}
\DontPrintSemicolon
\footnotesize
\KwIn{Graph $\mathcal G_{\mathrm c}$; local actors $\{\theta_n^0\}$; common shared-critic initialization $\psi^0$; local heads $\{\omega_n^0\}$; trigger, compression, and fusion parameters.}
Initialize
$
\widehat\psi_n^0=\widetilde\psi_n^0=\psi_n^0=\psi^0
$
and
$
\mathbf e_n^0=\mathbf0
$
for all $n$\;
\For{training round $t=0,1,\ldots$}{
    \ForEach{BS $n\in\mathcal N$ in parallel}{
        Collect the local rollout $\mathcal D_n^t$\;
        {Compute TD residuals, GAE advantages, and frozen return targets using
        \eqref{eq:td_residual_final}--\eqref{eq:lambda_return_final}}\;
        Perform local critic optimization to obtain
        $(\widetilde\psi_n^{t+1},\widetilde\omega_n^{t+1})$\;
        Update the local actor using \eqref{eq:actor_loss_final}\;
        {Compute
        $
        \bar Q_n^t
        $,
        $
        \bar I_n^t
        $,
        and
        $
        \{\bar\kappa_{nb}^t\}_{b\in\mathcal B_n}
        $ and exchange the scalar directional-relevance summaries with neighboring BSs}\;
        Evaluate the trigger, form
        $
        \boldsymbol\Delta_n^t
        $,
        and update
        $
        \widehat\psi_n^{t+1}
        $
        and
        $
        \mathbf e_n^{t+1}
        $\;
        \If{$\xi_n^t=1$}{
            {Transmit the nonzero critic coordinates and their indices to $\mathcal B_n$}\;
        }
    }
    \ForEach{BS $n\in\mathcal N$ in parallel}{
        Construct
        $
        \mathbf W^t
        $
        from
        \eqref{eq:symmetric_interference_score}--\eqref{eq:self_fusion_weight}\;
        {Update the shared critic parameters using
        \eqref{eq:critic_fusion_final}
        and retain}
        $
        \omega_n^{t+1}
        =
        \widetilde\omega_n^{t+1}
        $\;
    }
}
\end{algorithm}

\section{Simulation Results}
\vspace{-2mm}
\subsection{Simulation Setup}
\label{subsec:simulation_setup}

We evaluate FedCritic-MIMO in a strongly interference-coupled reuse-$1$
multi-cell massive-MIMO OFDMA downlink consistent with
Section~\ref{sec:system_model}. The network has $7$ BSs, $8$ UEs per BS,
$16$ subcarriers, $32$ BS antennas, and up to $3$ simultaneous spatial
streams per subcarrier. Consistent with the open and
disaggregated RAN setting, each BS is treated as an independently
deployable cell-level controller with local observations, local actor
execution, and private experience. FedCritic-MIMO does not use a central
trainer or parameter server; collaboration is limited to peer-to-peer
exchange of shared critic parameters. Large-scale gains follow the
lognormal model in Table~\ref{tab:simulation_parameters}, and small-scale
channels follow temporally correlated, spatially i.i.d. Rayleigh fading
with Gauss--Markov coefficient $\rho=0.55$. The radius-$2$ graph is used as the logical inter-controller critic-exchange graph for
critic exchange and fusion; SINR computation includes interference from
all co-channel BSs.

One environment step is one network-wide wireless slot in which all BSs
observe local states, select actions, receive rewards, and update channel
and virtual-queue states. Each policy update uses a rollout of $H=128$
joint steps; thus, $U$ updates correspond to $128U$ network-wide
interactions, or $32{,}000$ interactions for $U=250$. Interactions are
counted per wireless slot, not per BS.

All learning-based methods use the same local actor architecture, action
space, PPO hyperparameters, rollout budget, and teacher-guided warm start.
Behavior-cloned actors and pretrained critics are copied identically to all
compatible methods before method-specific training, so differences arise
only from information scope, critic parameterization, and critic
coordination. For FedCritic-MIMO, the actor and personalized
critic components remain local during training, while only the shared
critic component participates in peer-to-peer collaboration.

Each method is trained with six independent random seeds. During training,
validation on a fixed held-out channel set is used for checkpoint
selection; selected checkpoints are then evaluated without exploration or
gradient updates on unseen channel realizations. Unless otherwise stated,
learning curves show run-level means with $95\%$ confidence intervals, and
final bars and distributions are computed from the selected checkpoints on
the held-out evaluation set.

Reported communication overhead refers to training-side
critic-model traffic among distributed controllers. It does not represent
user-plane payload traffic, fronthaul traffic, or a specific standardized
Open RAN interface. The principal simulation and learning parameters are
summarized in Table~\ref{tab:simulation_parameters}.

\begin{table*}[t]
\centering
\caption{Main simulation and learning parameters.}
\label{tab:simulation_parameters}
\renewcommand{\arraystretch}{1}
\setlength{\tabcolsep}{4pt}
\footnotesize
\begin{tabular}{l l l l}
\toprule
\textbf{Parameter} & \textbf{Value} &
\textbf{Parameter} & \textbf{Value} \\
\midrule

\multicolumn{4}{c}{\textit{Network and channel}}\\
\midrule
Number of BSs, $N$ & $7$ &
UEs per BS, $|\mathcal M_n|$ & $8$\\

BS antennas, $L_n$ & $32$ &
Maximum spatial streams, $S_n$ & $3$\\

Subcarriers, $K$ & $16$ &
Frequency reuse & $1$\\

Per-BS power budget, $P_n$ & $1$ normalized unit &
Noise PSD, $N_0$ & $10^{-3}$\\

Temporal channel correlation, $\rho$ & $0.55$ &
Small-scale fading &
{Temporally correlated, spatially i.i.d.\ Rayleigh}\\

{Direct-link large-scale gain} &
{$\ln\beta^{\mathrm{dir}}\sim\mathcal N(-2.3,1.10^2)$} &
{Cross-link gain multiplier} &
{$3.0$}\\

Minimum-rate target, $R_{n,m}^{\min}$ &
{$1.9$ normalized units} &
Coordination graph &
{Radius-$2$ ring; all BSs contribute physical interference}\\

{Power-utilization levels, $\upsilon_n$} &
\{${0.2,0.4,0.6,0.8,1.0}$\} &
{Normalized RZF levels, $\bar\alpha$} &
\{${10^{-3},10^{-2},5{\times}10^{-2},10^{-1},5{\times}10^{-1}}$\}\\

\midrule
\multicolumn{4}{c}{\textit{Actor--critic training}}\\
\midrule
Discount factor, $\gamma$ & $0.99$ &
GAE parameter, $\lambda$ & $0.95$\\

Actor learning rate, $\eta_{\mathrm a}$ & $3\times10^{-4}$ &
Critic learning rate, $\eta_{\mathrm c}$ & $5\times10^{-4}$\\

PPO clipping, $\epsilon_\pi$ & $0.20$ &
Entropy coefficient, $\beta_H$ & $0.02$\\

Actor/critic epochs & $4/4$ &
Gradient-norm limit & $0.50$\\

{Rollout horizon/minibatch size} &
{$128/64$} &
{Maximum training updates} &
{$250$}\\

{Independent training runs} &
{$6$} &
{Evaluation} &
{$6$ validation and $30$ held-out channel seeds/run}\\

Queue normalization, $q_0$ & $10$ &
Reward scaling & $0.01$\\

\midrule
\multicolumn{4}{c}{\textit{Federated communication and fusion}}\\
\midrule
{Shared critic dimension, $d_c$} &
{$50,048$} &
Value/index precision & $32/16$ bits\\

{Overall compression budget, $\rho_c$} &
{$0.20$} &
{Per-layer ratio bounds} &
{$[0.10,0.25]$}\\

Initial/floor trigger thresholds &
{$\tau_{\mathrm{th},0}=0.020$, $\tau_{\mathrm{th},\min}=0.001$} &
Queue/interference coefficients &
$\alpha_Q=1.0$, $\alpha_I=1.5$\\

Local fusion weight, $\omega_{\mathrm{self}}$ & $0.55$ &
Exchanged parameters & Shared critic only\\
\bottomrule
\end{tabular}
\end{table*}

\vspace{-3mm}
\subsection{Baselines and Comparison Protocol}
\label{subsubsec:simulation_baselines}

We compare FedCritic-MIMO with representative heuristic, independent-learning, centralized-training, and communication-ablation baselines.

\begin{itemize}
\item \emph{Random}: feasible scheduling, stream activation, power, and RZF-regularization actions are selected randomly.

\item \emph{Greedy-MaxGain}: each BS schedules UEs with the largest instantaneous direct-channel gains and uses the maximum available power-fraction level.

\item \emph{Greedy-Queue}: each BS applies a queue-weighted channel-gain rule to account for long-term rate deficits without learning.

\item \emph{Greedy-IA-Queue}: the Greedy-Queue metric is further normalized by incoming and outgoing interference-risk terms, forming the strongest non-learning heuristic.

\item \emph{Strict-Independent-PPO}: each BS trains a purely local actor--critic using only own-cell channel, queue, occupancy, and measured-SINR information, without explicit neighbor features or outgoing-interference pricing.

\item \emph{No-Federation-IA-PPO}: each BS uses the same interference-aware observations, reward, actor, and critic architectures as FedCritic-MIMO, but without critic exchange. This isolates the benefit of serverless critic collaboration.

\item \emph{CTDE-MAPPO}: decentralized actors are trained with a centralized critic using concatenated BS-level critic features, while execution remains local.

\item \emph{Periodic-Full}: every BS exchanges its complete shared critic parameters with coordination neighbors at every training update.

\item \emph{Event-Uncompressed}: the proposed event rule determines communication times, but transmitted critic increments are uncompressed.

\item \emph{Proposed}: the complete FedCritic-MIMO framework with utility-aware event-triggered shared-critic exchange, adaptive layer-wise top-$k$ compression with error feedback, and balanced interference-aware serverless fusion.
\end{itemize}

\begin{figure}[!t] 
\centering
  \subfloat[]{
    \includegraphics[width=0.4\textwidth]{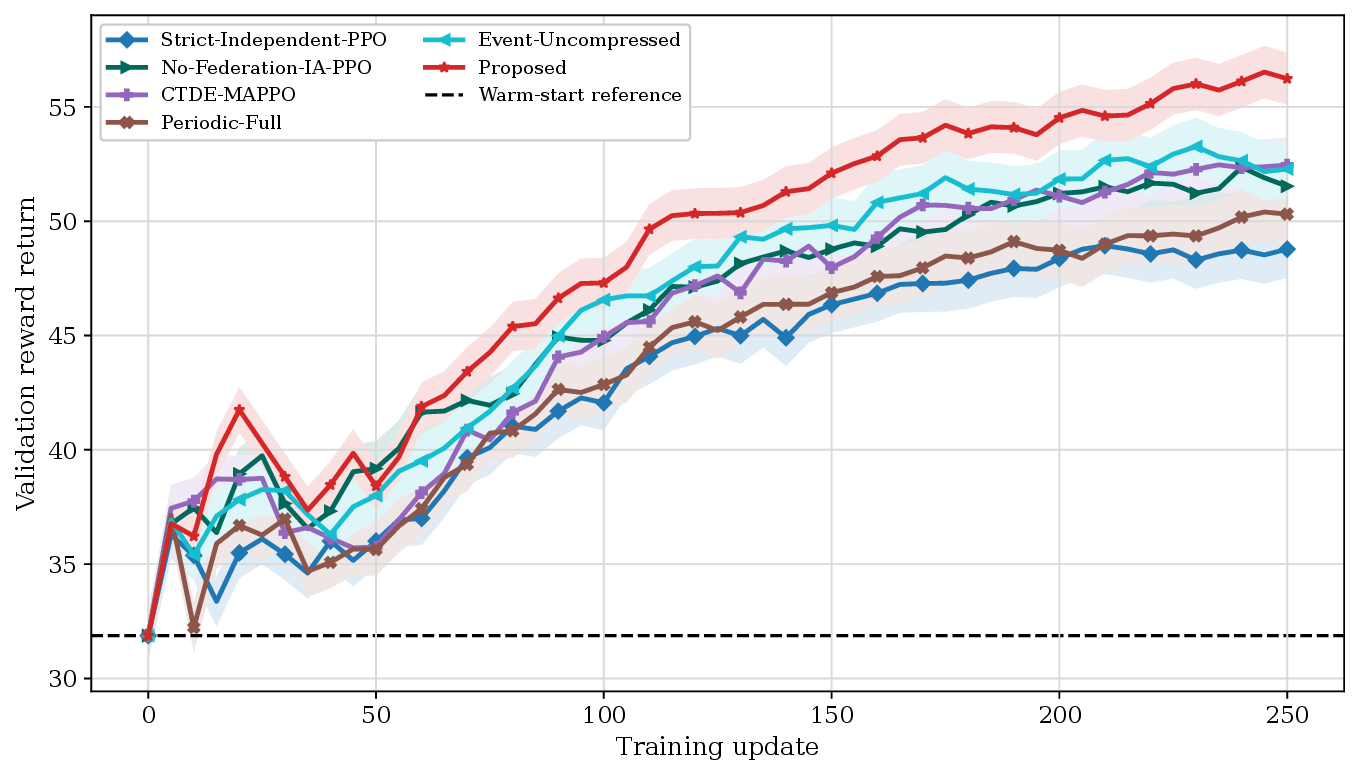}}
  \vspace*{-3mm}
  \hfill
  \subfloat[]{
    \includegraphics[width=0.4\textwidth]{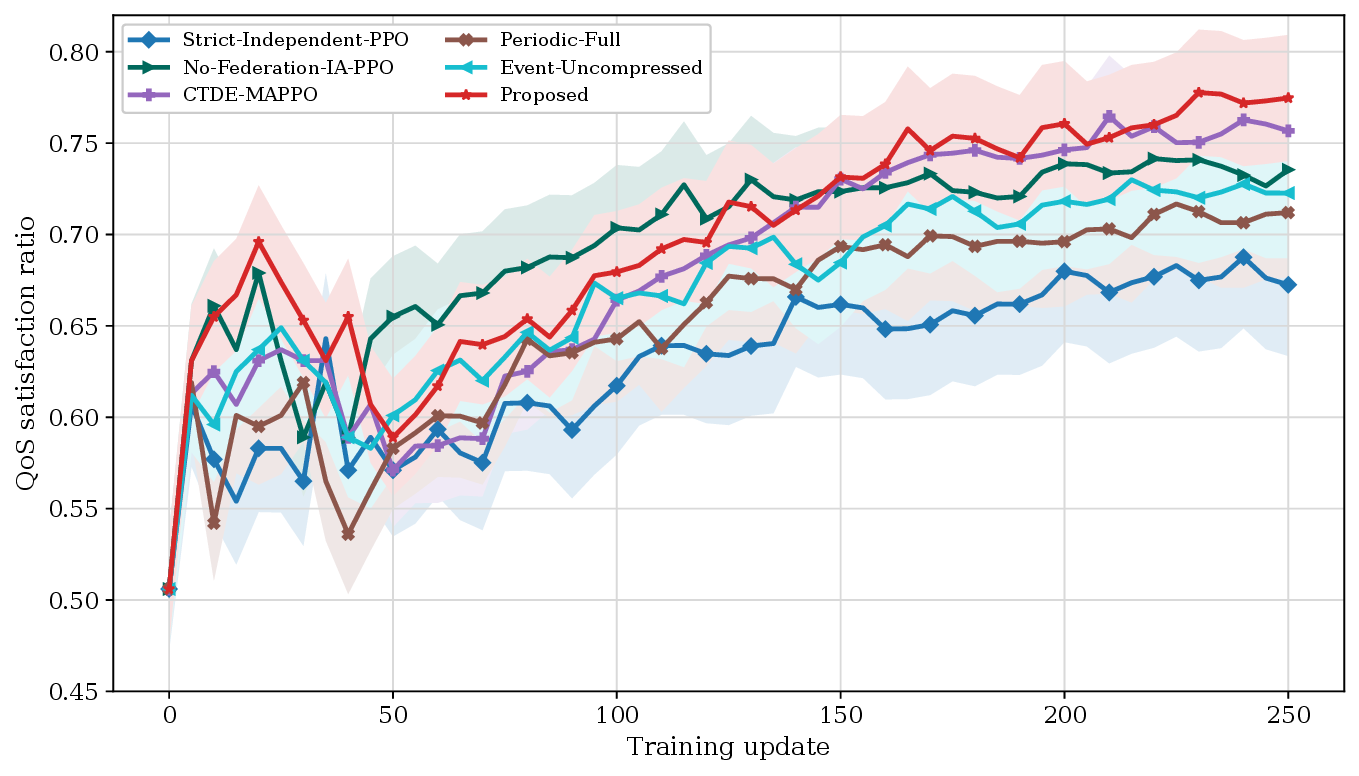}}
    \vspace*{-3mm}
  \hfill
  \subfloat[]{
    \includegraphics[width=0.4\textwidth]{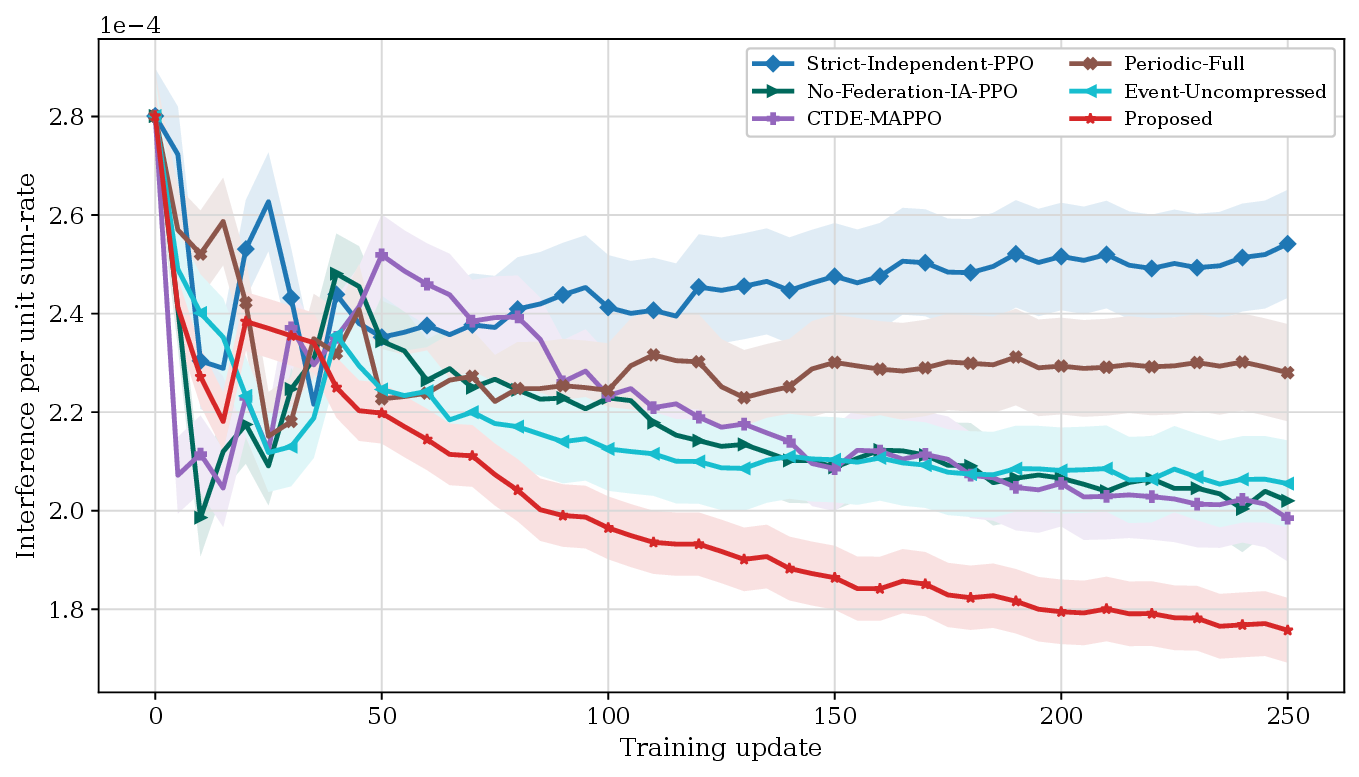}}
    \vspace*{-2mm}
\caption{Validation learning dynamics over training updates: (a) reward, (b) QoS satisfaction, and (c) interference cost per unit sum rate.}
  \vspace{-3mm}
\label{fig:learning_dynamics}
\end{figure}
\vspace{-4mm}
\subsection{Learning Dynamics}
\label{subsec:learning_dynamics}
Fig.~\ref{fig:learning_dynamics} reports the evolution of the validation reward,
the QoS-satisfaction metric, and the interference cost per unit sum rate. All
learning-based methods improve substantially beyond the common warm-start
reference, confirming that the subsequent policy updates provide material gains
over the teacher-guided initialization. As Fig.~\ref{fig:learning_dynamics} (a) shows, the proposed method separates from the
other methods after approximately the first one hundred updates and reaches a
validation reward of about $56.3$ at update $250$, compared with approximately
$52.5$, $52.0$, $51.5$, $50.3$, and $48.8$ for CTDE-MAPPO,
Event-Uncompressed, No-Federation-IA-PPO, Periodic-Full, and
Strict-Independent-PPO, respectively.

The QoS trajectories in Fig.~\ref{fig:learning_dynamics} (b), exhibit a similar, although less separated, ordering. At
the final displayed update, the proposed method reaches approximately $0.78$,
while CTDE-MAPPO, No-Federation-IA-PPO, Event-Uncompressed, Periodic-Full, and
Strict-Independent-PPO reach approximately $0.76$, $0.74$, $0.73$, $0.71$, and
$0.67$, respectively. The confidence bands of the proposed method and
CTDE-MAPPO overlap near the end of training; hence, the figure supports the
conclusion that the proposed decentralized method matches the strongest
centralized-training baseline in QoS, rather than establishing a statistically
significant QoS advantage from the learning curves alone.

As shown in Fig.~\ref{fig:learning_dynamics} (c), the proposed method also produces the most consistent reduction in the
interference-per-rate metric, for which lower values are preferable. It
decreases from approximately $2.8\times10^{-4}$ at initialization to about
$1.8\times10^{-4}$, whereas the strongest competing learning methods finish
between approximately $2.0\times10^{-4}$ and $2.1\times10^{-4}$. Thus, the
reward improvement is accompanied by more interference-efficient operation,
rather than being obtained solely by increasing the delivered rate without
controlling the resulting inter-cell interference. Because the proposed reward
curve remains mildly increasing at the last displayed update, these results
should be interpreted as finite-budget performance rather than evidence of
empirical convergence.
\vspace{-0.2mm}
\begin{figure}[!t]
    \centering
    \includegraphics[width=0.8\linewidth]{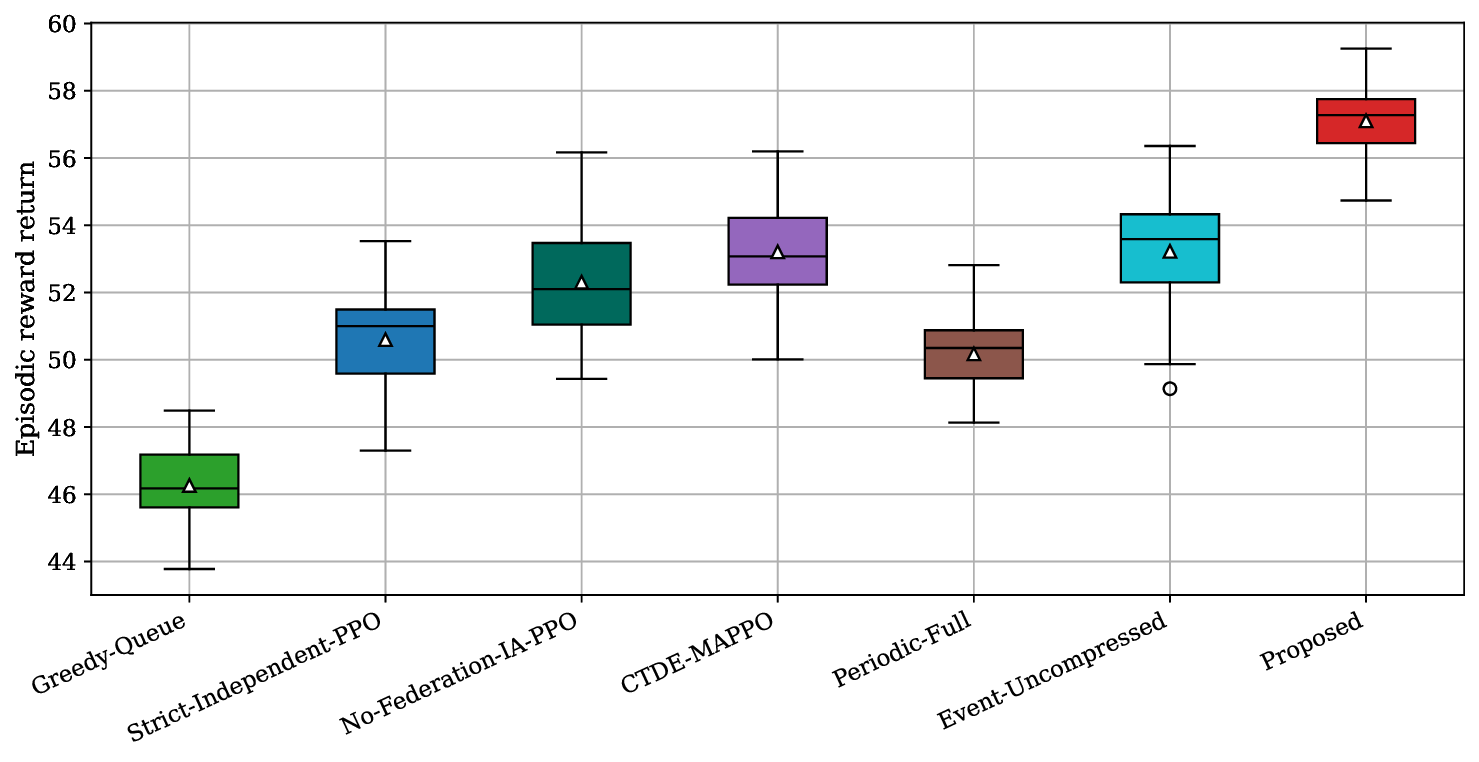}
    \caption{Held-out episodic reward distribution.}
    \vspace{-2mm}
    \label{fig:heldout_reward}
\end{figure}
\begin{figure}[!t]
    \centering
    \includegraphics[width=0.8\linewidth]{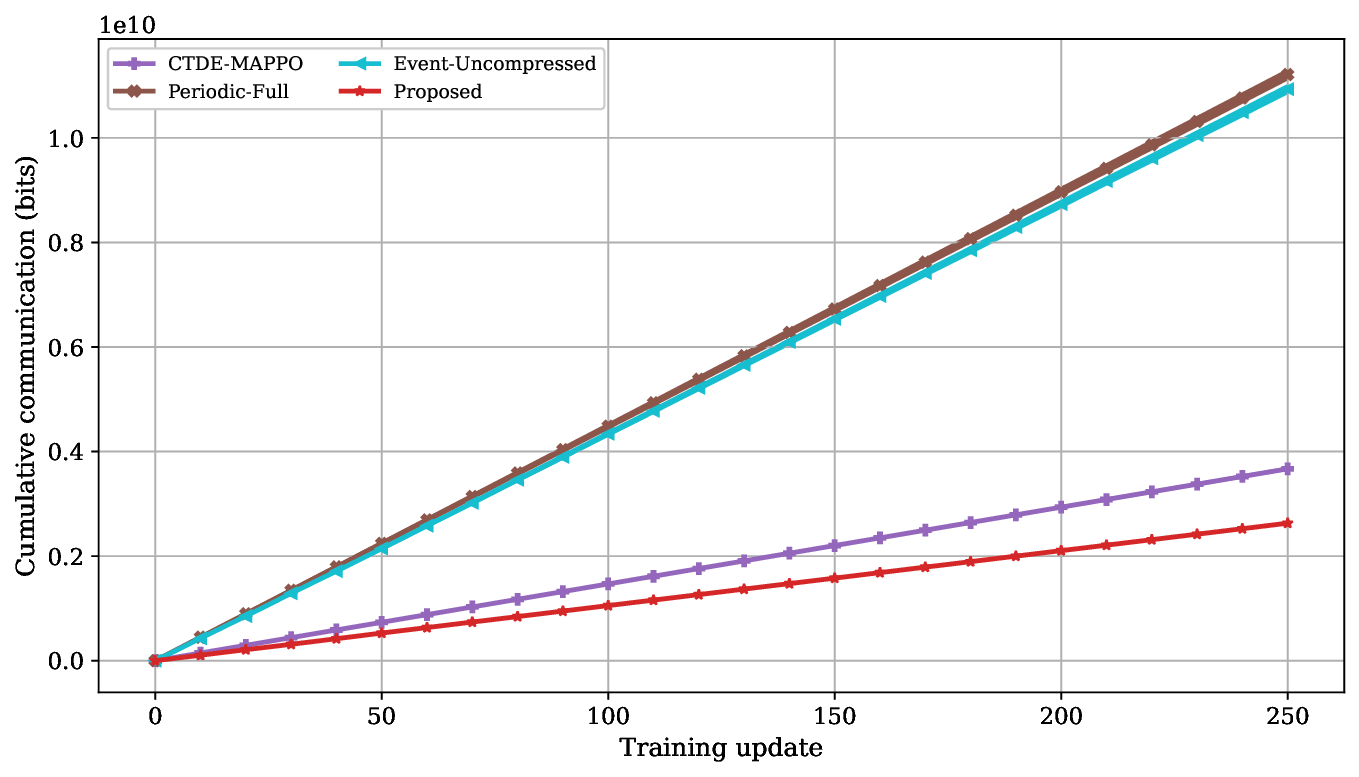}
\caption{Cumulative training-side communication overhead.}
    \vspace{-2mm}
    \label{fig:communication_overhead}
\end{figure}

\subsection{Held-Out Reward Performance}
\label{subsec:reward_results}

Fig.~\ref{fig:heldout_reward} evaluates the selected policies on channel
realizations not used for training or checkpoint selection. The proposed method
attains a mean episodic reward of approximately $57.0$. Event-Uncompressed and
CTDE-MAPPO achieve approximately $53.3$ and $53.1$, respectively, while
No-Federation-IA-PPO reaches approximately $52.2$. Strict-Independent-PPO and
Periodic-Full attain approximately $50.4$ and $50.3$, and Greedy-Queue reaches
approximately $46.4$. Accordingly, the proposed method improves the mean
held-out reward by about $6.8\%$ relative to Event-Uncompressed, $7.2\%$
relative to CTDE-MAPPO, and $9.1\%$ relative to
No-Federation-IA-PPO.

The proposed distribution is shifted to the right and has a comparatively
compact interquartile range. This indicates that the gain is present across the
held-out realizations rather than being generated by a small number of
exceptionally favorable episodes. Nevertheless, episode-level samples generated
by the same trained policy are not independent training replicates. Therefore,
claims of statistical significance should be based on paired comparisons across
the independent training runs, while the boxplot is used as a descriptive
representation of held-out variability.
\vspace{-3mm}
\subsection{Communication Efficiency}
\label{subsec:communication_results}

Fig.~\ref{fig:communication_overhead} reports cumulative training-side
communication under the adopted accounting rule. At update $250$,
Periodic-Full and Event-Uncompressed require approximately $11.2$ and
$11.0$~Gbits, respectively, whereas CTDE-MAPPO and the proposed method require
approximately $3.65$ and $2.65$~Gbits. The proposed method therefore reduces
the reported communication by approximately $76\%$ relative to both
uncompressed distributed critic-exchange baselines and by approximately $27\%$
relative to CTDE-MAPPO.

The small difference between Periodic-Full and Event-Uncompressed shows that
the selected event trigger remains active during most training rounds.
Consequently, the measured reduction relative to the uncompressed distributed
baselines is attributable primarily to layer-wise sparse critic exchange and
error feedback, rather than to a large decrease in the number of communication
rounds. The experiment therefore demonstrates the effectiveness of compressed
exchange at the selected performance-oriented operating point, but does not by
itself isolate the gain attributable to event triggering. Such isolation would
require a periodic-compressed baseline with the same sparsification budget.
Methods without critic coordination have zero federated-model communication,
but their lower reward, QoS, and interference efficiency show the performance
cost of eliminating collaboration entirely.

\subsection{Long-Term QoS and Mean SINR}
\label{subsec:qos_sinr_results}

Fig.~\ref{fig:QoS_ratio} compares the held-out QoS metric and mean SINR. Fig.~\ref{fig:QoS_ratio} (a) shows that the
proposed method obtains a QoS-satisfaction ratio of approximately $0.78$,
closely followed by CTDE-MAPPO at approximately $0.77$.
No-Federation-IA-PPO and Event-Uncompressed achieve approximately $0.74$ and
$0.73$, while Periodic-Full and Strict-Independent-PPO attain approximately
$0.70$ and $0.69$. The overlapping error bars of the proposed method and
CTDE-MAPPO again indicate comparable QoS performance; the more pronounced
benefit of the proposed method appears in the interference-related metrics.

In Fig.~\ref{fig:QoS_ratio} (b), the proposed method reaches a mean SINR of approximately
$-2.3$~dB, compared with approximately $-4.0$~dB for Event-Uncompressed,
$-4.8$~dB for CTDE-MAPPO, $-5.1$~dB for No-Federation-IA-PPO, $-5.6$~dB for
Periodic-Full, and $-5.9$~dB for Strict-Independent-PPO. Hence, the proposed
method improves the mean SINR by approximately $1.7$~dB relative to
Event-Uncompressed and by approximately $2.5$~dB relative to CTDE-MAPPO. The
negative absolute SINR values are consistent with the deliberately
interference-limited reuse-$1$ operating regime; the relevant observation is
the substantial relative improvement achieved through coordinated scheduling,
power control, and beamforming.

\vspace{-4mm}
\subsection{User-Rate Distribution}
\label{subsec:rate_distribution_results}

Fig.~\ref{fig:rate_cdf} presents the empirical CDF of the per-UE time-average
held-out rate. The proposed curve is shifted to the right of those of the
learning and heuristic baselines over most of the distribution. The separation
is particularly visible in the lower and median portions of the CDF, indicating
that the proposed method improves not only the average network outcome but also
the rates experienced by comparatively weak UEs. The curves approach one
another in the upper tail, suggesting that the principal benefit is improved
service to disadvantaged and typical users rather than an isolated increase in
the largest UE rates.

The vertical line at $R_{n,m}^{\min}=1.9$ identifies the long-term target. The
rate CDF and the reported QoS bar must be interpreted using explicitly distinct
definitions if the latter is computed as a time average of instantaneous
UE-slot satisfaction indicators. If the QoS bar is intended instead to denote
the fraction of UEs whose time-average rate satisfies the long-term constraint,
it must equal
$1-\widehat F_{\bar R}(R_{n,m}^{\min})$ and should be recomputed directly from
the same samples used for Fig.~\ref{fig:rate_cdf}. This distinction is required
to avoid conflating instantaneous service reliability with satisfaction of the
long-term average-rate constraint.

\begin{figure}[!t]
    \centering
  \subfloat[]{
    \includegraphics[width=0.8\linewidth]{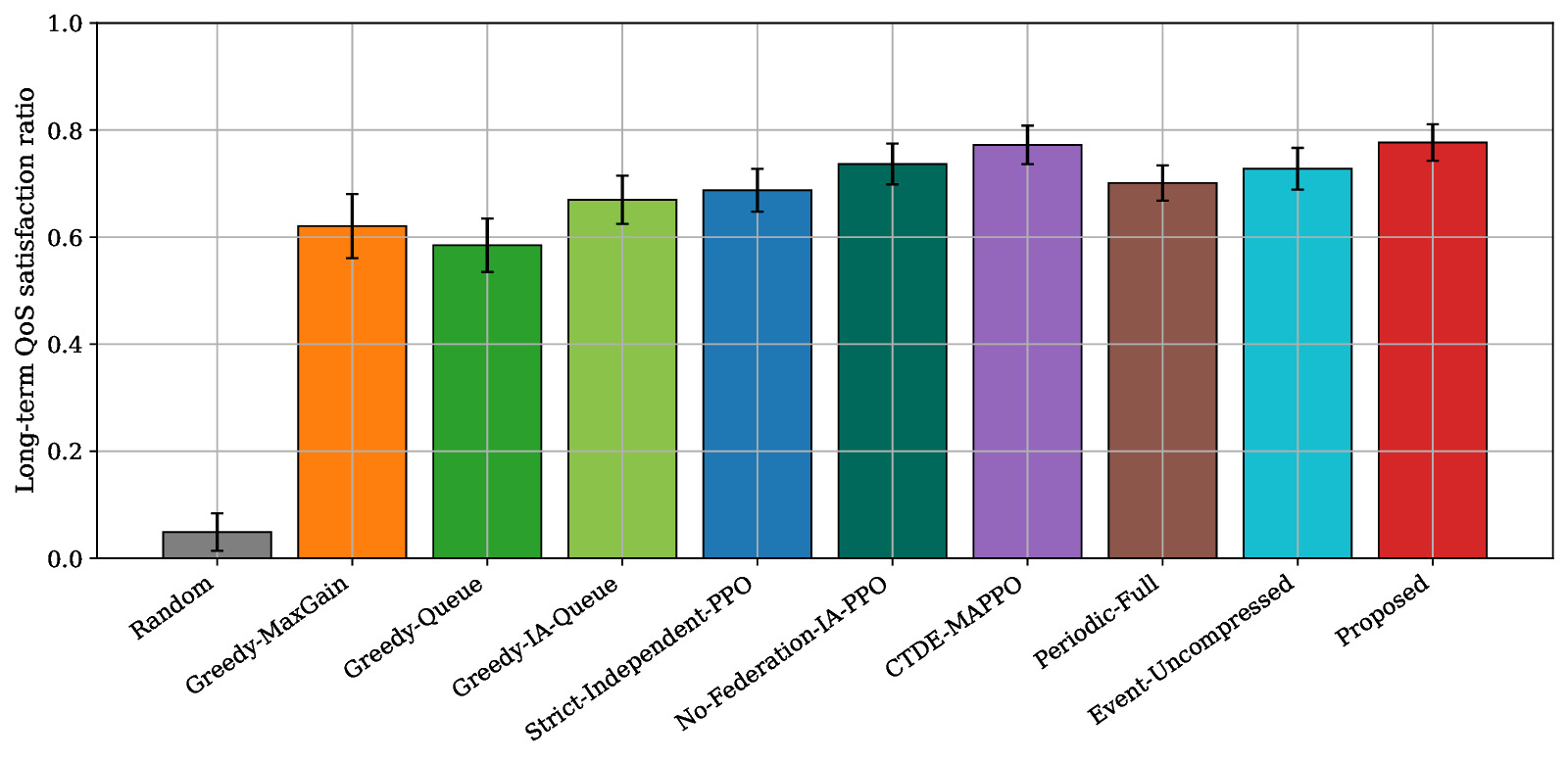}}
    \vspace{-2mm}
    \hfill
  \subfloat[]{\includegraphics[width=0.8\linewidth]{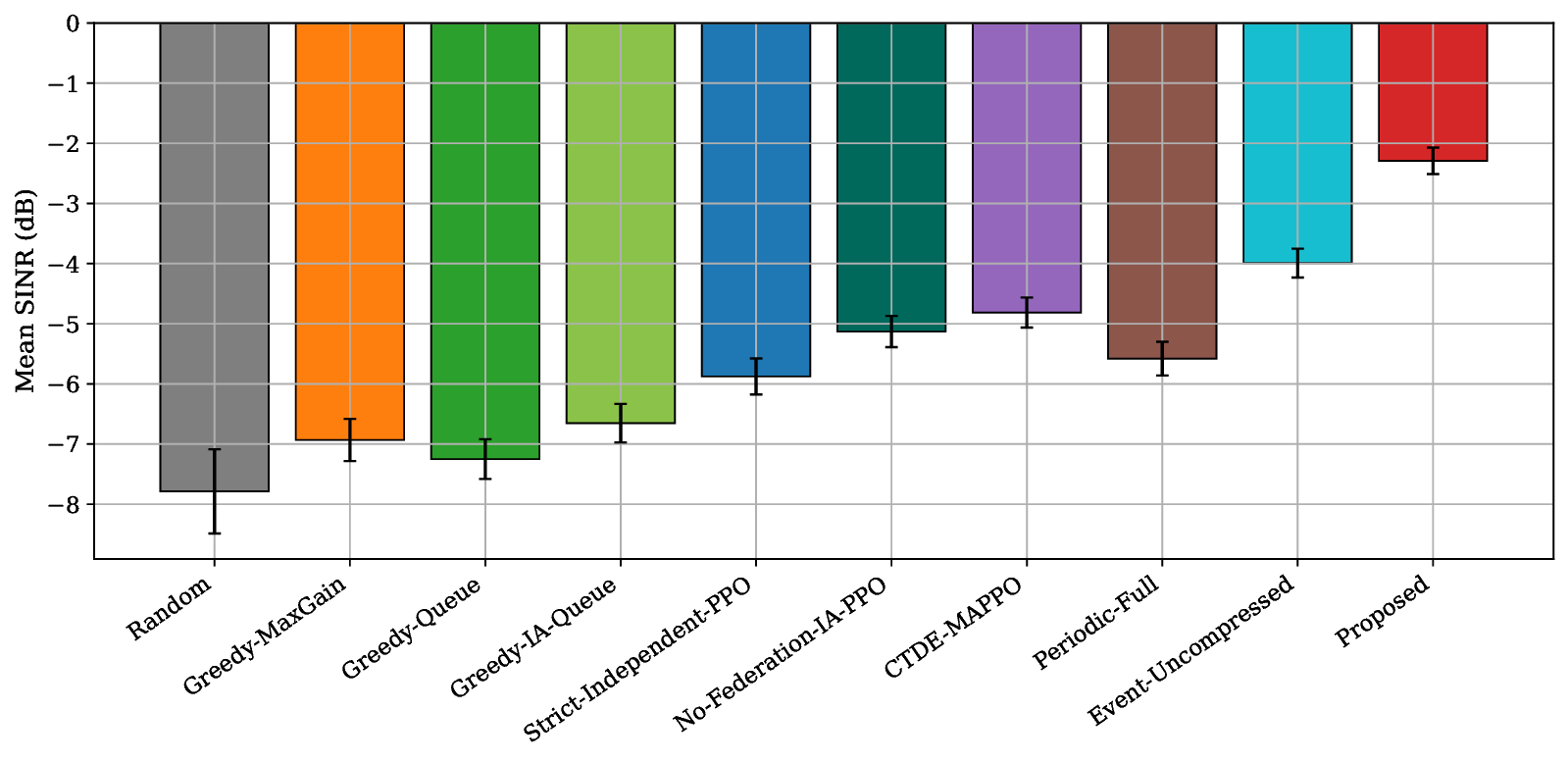}}
   \caption{Held-out QoS and SINR performance of the selected checkpoints: (a) long-term QoS-satisfaction ratio and (b) mean SINR.}
   \vspace{-3mm}
    \label{fig:QoS_ratio}
\end{figure}

\begin{figure}[!t]
\centering
    \includegraphics[width=0.85\linewidth]{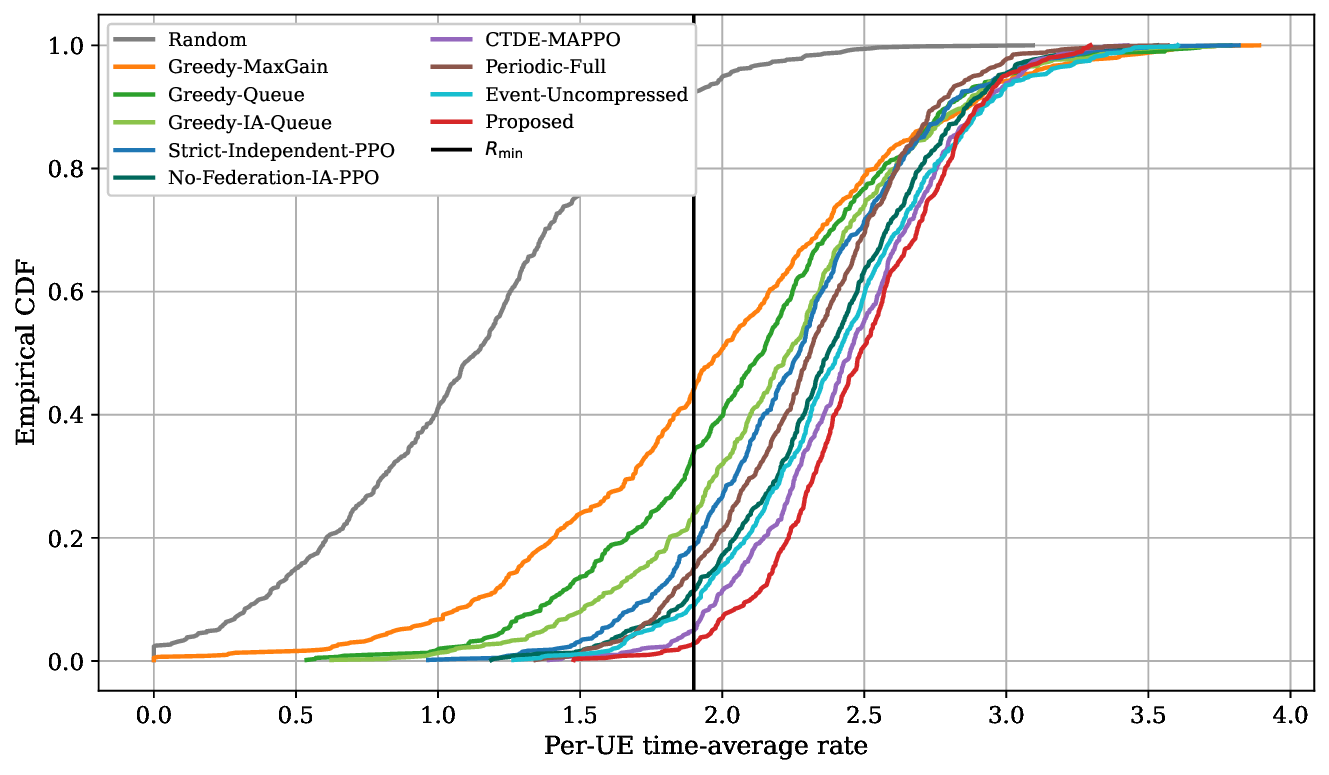}
    \caption{Empirical CDF of per-UE time-average held-out rates.}
    \vspace{-2mm}
    \label{fig:rate_cdf}
\end{figure}

\begin{figure}[!t]
    \centering
    \includegraphics[width=0.8\linewidth]{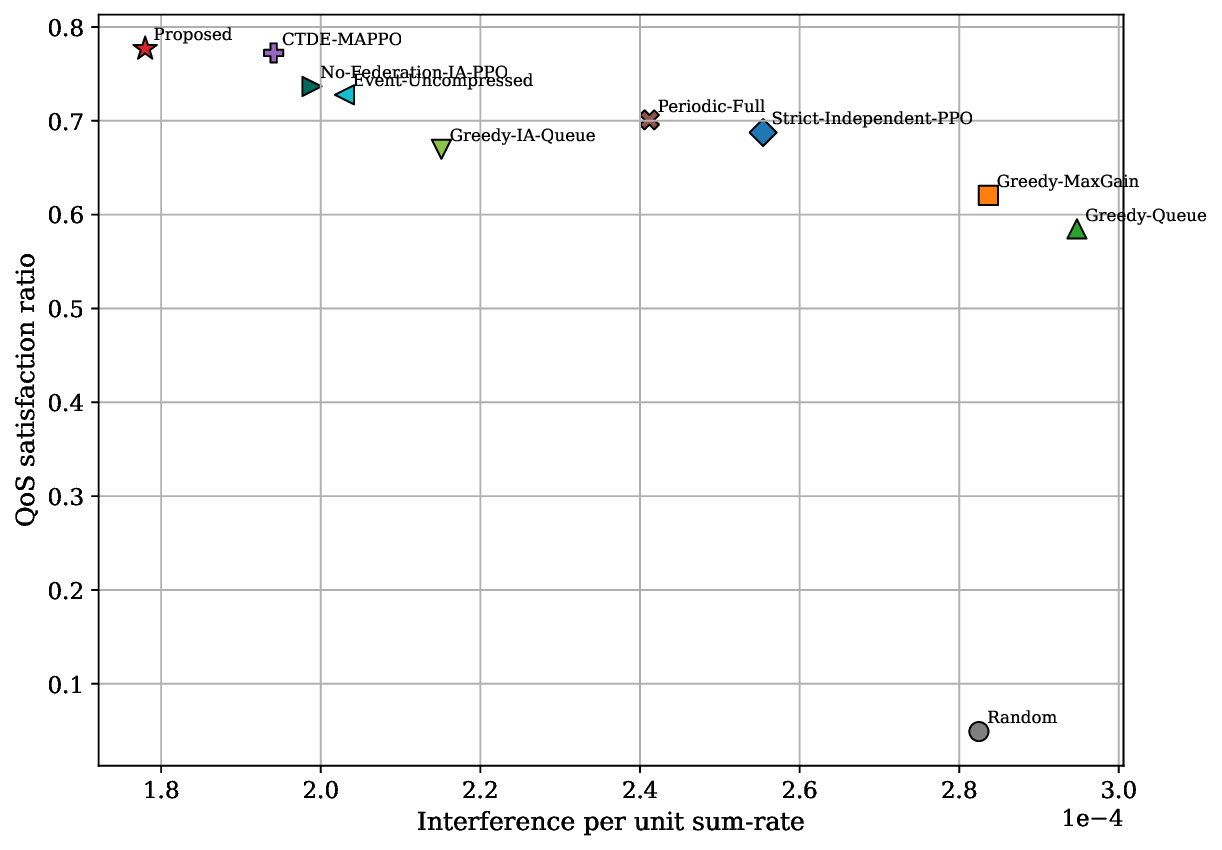}
\caption{QoS--interference operating points of the evaluated methods.}
\label{fig:qos_interference}
\end{figure}
\vspace{-4mm}
\subsection{Interference Management}
\label{subsec:interference_results}

The interference-per-rate results provide a normalized view of how effectively
each policy converts the interference it creates into useful throughput. The
proposed method attains approximately $1.78\times10^{-4}$, compared with
$1.93\times10^{-4}$ for CTDE-MAPPO, $1.98\times10^{-4}$ for
No-Federation-IA-PPO, $2.03\times10^{-4}$ for Event-Uncompressed, and
$2.40\times10^{-4}$ for Periodic-Full. These values correspond to reductions
of approximately $8\%$, $10\%$, $12\%$, and $26\%$, respectively.

Fig.~\ref{fig:qos_interference} summarizes the joint QoS--interference
operating points. The preferred region is the upper-left corner, representing
larger QoS satisfaction and smaller interference cost. The proposed method
occupies the most favorable displayed point. CTDE-MAPPO provides nearly the
same QoS ratio but at a visibly larger interference cost, whereas
No-Federation-IA-PPO and Event-Uncompressed sacrifice both QoS and
interference efficiency. The comparison with Strict-Independent-PPO and the
greedy baselines further shows that explicit interference awareness alone is
insufficient to attain the proposed operating point; collaborative critic
learning contributes additional value.

\vspace{-4mm}
\subsection{Overall Operating Point}
\label{subsec:overall_operating_point}

Taken together, the figures indicate that the proposed method provides the
strongest performance--communication operating point among the coordinated
learning methods considered. It achieves the largest validation and held-out
reward, the highest mean SINR, the smallest interference cost per unit sum
rate, and a QoS level comparable to the best centralized-training baseline.
Relative to uncompressed distributed critic exchange, these benefits are
obtained with an approximately $76\%$ reduction in reported training-side
communication. Relative to the non-federated learning baselines, the proposed
method uses additional training communication but obtains materially better
reward, user-rate, QoS, and interference outcomes. The result should therefore
be interpreted as a favorable joint operating point, not as a claim that the
proposed method minimizes communication in isolation.

\section{Conclusion}
\label{sec:conclusion}
This paper proposed FedCritic-MIMO, a communication-efficient serverless federated critic learning framework for massive-MIMO resource control in open and disaggregated 6G RANs. The framework enables independently deployable cell-level controllers to coordinate scheduling, power allocation, and beamforming without centralized trajectory collection, actor sharing, or parameter-server aggregation. By keeping local actors and personalized critic components private while exchanging only compatible shared critic parameters, FedCritic-MIMO preserves local control autonomy and supports collaboration over the physical interference graph. FedCritic-MIMO combines utility-aware triggering, adaptive layer-wise top-$k$ compression with error feedback, and balanced interference-aware fusion to make peer-to-peer critic collaboration both communication- efficient and radio-aware. We also established conditional finite-time stationarity and consensus guarantees for the balanced shared-critic recursion under a fixed-policy, frozen-target critic-regression model. Simulation results showed that FedCritic-MIMO achieves the most favorable performance--communication tradeoff among the considered baselines, improving throughput, user-rate distribution, SINR, QoS satisfaction, and interference efficiency while substantially reducing critic-model traffic. Future work will consider asynchronous peer coordination, mobility-aware interference graphs, imperfect CSI, and larger heterogeneous deployments. 


\vspace{-1mm}
\appendices
\renewcommand{\theequation}
{\thesection.\arabic{equation}}
\section{Proof of Lemma~\ref{lem:tracking_error}}
\label{app:proof_lem1}
\setcounter{equation}{0}

\begin{IEEEproof}
If $\xi_n^t=0$, then $\Gamma_n^t<\tau_{\mathrm{th}}$. Since
$1+\alpha_Q\bar Q_n^t+\alpha_I\bar I_n^t\geq1$ and
$\mathbf e_n^{t+1}=\mathbf d_n^t$, \eqref{eq:trigger_score} gives
$\|\mathbf e_n^{t+1}\|_2
<
\tau_{\mathrm{th}}\big(\|\widehat\psi_n^t\|_2+\epsilon_{\mathrm{tr}}\big)$.
If $\xi_n^t=1$, then
$\mathbf e_n^{t+1}
=\mathbf d_n^t-\mathcal C_n^t(\mathbf d_n^t)$, and
\eqref{eq:compressor_contraction} yields
$\|\mathbf e_n^{t+1}\|_2^2
\leq(1-\delta_c)\|\mathbf d_n^t\|_2^2$.
These two cases prove \eqref{eq:tracking_error_bound}.
\end{IEEEproof}
\vspace{-3mm}
\section{Proof of Lemma~\ref{lem:average_preservation}}
\label{app:proof_lem2}
\begin{IEEEproof}
Because $\mathbf W^t$ is doubly stochastic,
$(\mathbf I_N-(\mathbf W^t)^T)\mathbf1=\mathbf0$.
Equation~\eqref{eq:communication_perturbation} therefore gives
$\mathbf R^t\mathbf1=\mathbf0$. Right-multiplying
\eqref{eq:perturbed_critic_recursion} by $\mathbf1/N$ then gives
\eqref{eq:average_critic_recursion}. Finally,
$\|\mathbf I_N-(\mathbf W^t)^T\|_2\leq2$ for a symmetric stochastic
matrix; hence,
$\frac1N\mathbb E\|\mathbf R^t\|_{\mathrm F}^2
\leq
\frac4N\mathbb E
\|\widetilde{\boldsymbol\Psi}^{t+1}
-\widehat{\boldsymbol\Psi}^{t+1}\|_{\mathrm F}^2
=
4\varepsilon_t^2$,
which proves \eqref{eq:perturbation_tracking_bound}.
\end{IEEEproof}
\vspace{-5mm}
\section{Proof of Theorem~\ref{thm1}}
\label{app:proof_thm1}
\begin{IEEEproof}
Let
$\mathbf J\triangleq\frac1N\mathbf1\mathbf1^T$,
$\boldsymbol\Pi \triangleq\mathbf I_N-\mathbf J$, and
$\mathbf Z^t\triangleq\boldsymbol\Psi^t\boldsymbol\Pi$. Then
$N^{-1}\|\mathbf Z^t\|_{\mathrm F}^2=\mathcal E_\psi^t$.
With
$\mathbf A^t\triangleq\mathbf W^t-\mathbf J$,
double stochasticity gives
$\mathbf W^t\boldsymbol\Pi=\boldsymbol\Pi\mathbf W^t=\mathbf A^t$ and
$\|\mathbf A^t\|_2\leq\lambda_{\mathrm W}$. Using
$\mathbf R^t\boldsymbol\Pi=\mathbf R^t$, the disagreement recursion is
\vspace{-2mm}
\begin{equation}
\mathbf Z^{t+1}
=
\big(\mathbf Z^t-\eta_{\mathrm c}\mathbf G^t\boldsymbol\Pi\big)
(\mathbf A^t)^T+\mathbf R^t.
\label{eq:app_disagreement_recursion_compact}
\end{equation}

Let
$\mathbf U^t=[\nabla F_1(\psi_1^t),\ldots,\nabla F_N(\psi_N^t)]$.
By smoothness, bounded heterogeneity, and
\eqref{eq:local_noise_variance},
\vspace{-2mm}
\begin{equation}
\frac1N\mathbb E
\|\mathbf G^t\boldsymbol\Pi\|_{\mathrm F}^2
\leq
c_0\!\left(
L^2\mathbb E[\mathcal E_\psi^t]
+\sigma_g^2+\zeta^2
\right)
\label{eq:app_gradient_disagreement_compact}
\end{equation}
for a constant $c_0>0$. Applying Young's inequality to
\eqref{eq:app_disagreement_recursion_compact}, using
\eqref{eq:perturbation_tracking_bound}, and choosing
$\eta_{\mathrm c}\leq\eta_{\max}$ sufficiently small relative to
$L$ and $1-\lambda_{\mathrm W}$, yields constants $c_1,c_2>0$ and
$\rho_{\mathrm{dis}}\in(0,1)$ with $1-\rho_{\mathrm{dis}}\geq c_1(1-\lambda_{\mathrm W})$ such that
\vspace{-2mm}
\begin{equation}
\mathbb E[\mathcal E_\psi^{t+1}]
\leq
\rho_{\mathrm{dis}}\,\mathbb E[\mathcal E_\psi^t]
+
\frac{c_2\eta_{\mathrm c}^2}{1-\lambda_{\mathrm W}}
(\sigma_g^2+\zeta^2)
+
\frac{c_2}{1-\lambda_{\mathrm W}}\varepsilon_t^2.
\label{eq:app_consensus_recursion_compact}
\end{equation}
Since $\mathcal E_\psi^0=0$, summing
\eqref{eq:app_consensus_recursion_compact} and using
$\sum_{j\geq0}\rho_{\mathrm{dis}}^j=(1-\rho_{\mathrm{dis}})^{-1}$ gives
\vspace{-2mm}
\begin{equation}
\frac1T\sum_{t=0}^{T-1}\mathbb E[\mathcal E_\psi^t]
\leq
C_5
\frac{\eta_{\mathrm c}^2(\sigma_g^2+\zeta^2)}
{(1-\lambda_{\mathrm W})^2}
+
\frac{C_6}{(1-\lambda_{\mathrm W})^2}
\frac1T\sum_{t=0}^{T-1}\varepsilon_t^2,
\label{eq:app_consensus_final_compact}
\end{equation}
which proves \eqref{eq:critic_consensus_bound}.

It remains to bound the aggregate objective. By
\eqref{eq:average_critic_recursion},
$\bar\psi^{t+1}=\bar\psi^t-\eta_{\mathrm c}\bar g^t$.
Define
$\bar h^t\triangleq\frac1N\sum_{n=1}^{N}\nabla F_n(\psi_n^t)$ and
$\mathbf b^t\triangleq\bar h^t-\nabla F(\bar\psi^t)$. Smoothness and Jensen's inequality imply
$\|\mathbf b^t\|_2^2\leq L^2\mathcal E_\psi^t$.
Moreover, Assumption~1 gives
$\mathbb E[\bar g^t\mid\mathcal F_t]=\bar h^t$ and
$\mathbb E[\|\bar g^t-\bar h^t\|_2^2\mid\mathcal F_t]
\leq\sigma_{\mathrm{av}}^2$.
The descent lemma, followed by Young's inequality, therefore gives, for
a sufficiently small $\eta_{\mathrm c}$,
\vspace{-2mm}
\begin{align}
\mathbb E[F(\bar\psi^{t+1})]
\leq{}&
\mathbb E[F(\bar\psi^t)]
-\frac{\eta_{\mathrm c}}{2}
\mathbb E\|\nabla F(\bar\psi^t)\|_2^2
\nonumber\\
&+
c_3\eta_{\mathrm c}L^2\mathbb E[\mathcal E_\psi^t]
+
c_4L\eta_{\mathrm c}^2\sigma_{\mathrm{av}}^2.
\label{eq:app_descent_compact}
\end{align}
Summing \eqref{eq:app_descent_compact}, using the lower bound
$F^\star$, dividing by $\eta_{\mathrm c}T$, and substituting
\eqref{eq:app_consensus_final_compact} yields
\eqref{eq:critic_stationarity_bound} after absorbing numerical and
smoothness constants into $C_1,\ldots,C_4$. Together with
\eqref{eq:app_consensus_final_compact}, this proves the theorem.
\end{IEEEproof}
\vspace{-5mm}
\section{Proof of Corollary~\ref{cor:critic_rate}}
\label{app:proof_cor1}
\begin{IEEEproof}
Uniformity and independence of $J_T$ imply
$\mathbb E\|\nabla F(\bar\psi^{J_T})\|_2^2
=
\frac1T\sum_{t=0}^{T-1}
\mathbb E\|\nabla F(\bar\psi^t)\|_2^2$. 
Substituting
$\eta_{\mathrm c}=c_\eta/\sqrt T$ and
\eqref{eq:tracking_rate_condition} into
\eqref{eq:critic_stationarity_bound} gives, respectively,
$\mathcal O(T^{-1/2})$, $\mathcal O(T^{-1/2})$, $\mathcal O(T^{-1})$, and
$\mathcal O(\log T/T)$ for its four right-hand-side terms, proving
\eqref{eq:critic_convergence_rate}. For fixed $\eta_{\mathrm c}$,
taking the limit superior in \eqref{eq:critic_stationarity_bound}
and using
$\limsup_{T\to\infty}T^{-1}\sum_{t<T}\varepsilon_t^2
\leq\bar\varepsilon^2$
gives the stated stationarity neighborhood.
\end{IEEEproof}


\end{document}